\documentclass[10pt,twoside,letterpaper]{extarticle}
\usepackage{lipsum}
\usepackage{amsthm,amsmath,amssymb,amsfonts,bm,dsfont}
\def\1{{\mathds{1}}}

\def\vx{{\boldsymbol{x}}}

\DeclareMathAlphabet{\mathsfit}{\encodingdefault}{\sfdefault}{m}{sl}
\SetMathAlphabet{\mathsfit}{bold}{\encodingdefault}{\sfdefault}{bx}{n}

\usepackage{booktabs}
\usepackage{multirow}
\usepackage{graphicx}
\usepackage{subfigure}
\usepackage{multirow}
\usepackage{float}
\usepackage{wrapfig}
\usepackage[style=alphabetic,natbib=true,backend=biber,backref=true]{biblatex} 
\usepackage{algorithm}
\usepackage{algorithmic}
\usepackage{color}

\usepackage[colorlinks,linkcolor=RawSienna,urlcolor=cyan,citecolor=teal,bookmarks=False,breaklinks]{hyperref}
\usepackage[dvipsnames]{xcolor}
\usepackage{times,url,caption}

\usepackage[shortlabels]{enumitem}
\usepackage{enumerate}
\usepackage{thmtools}
\usepackage{thm-restate}

\newtheorem{theorem}{Theorem}

\newtheorem{proposition}[theorem]{Proposition}

\theoremstyle{definition}

\DefineBibliographyStrings{english}{%
	backrefpage = {Cited on page},
	backrefpages = {Cited on pages},
}

\author{
	\normalsize	\textbf{Yi-Xiao He}, \ \textbf{Shen-Huan Lyu}, \ \textbf{Yuan Jiang}\\ 
	\normalsize	National Key Laboratory for Novel Software Technology,\\
	\normalsize	Nanjing University, Nanjing, 210023, China.\\ 
	{\normalsize \tt \{heyx,lvsh,jiangy\}@lamda.nju.edu.cn}
}

\date{}

\title{\vspace{-1.0cm}\hrule height 4pt\vspace{0.5cm}\textbf{Interpreting Deep Forest through Feature Contribution and MDI Feature Importance}\vspace{0.5cm}\hrule height 1.5pt \vspace{0.3cm}}
\allowdisplaybreaks[4]

\usepackage{abstract}

\usepackage{geometry}
\usepackage{textpos}

\providecommand{\keywords}[1]
{
	\normalsize	
	\textbf{\textit{Keywords: }} #1
}

\begin{document}
	
\maketitle
\thispagestyle{empty}

\begin{textblock}{10}(-0.0,9.70)\small
	Under review. 
\end{textblock}

\begin{abstract}\normalsize
	Deep forest is a non-differentiable deep model which has achieved impressive empirical success across a wide variety of applications, especially on categorical/symbolic or mixed modeling tasks. Many of the application fields prefer explainable models, such as random forests with feature contributions that can provide local explanation for each prediction, and Mean Decrease Impurity (MDI) that can provide global feature importance. However, deep forest, as a cascade of random forests, possesses interpretability only at the first layer. From the second layer on, many of the tree splits occur on the new features generated by the previous layer, which makes existing explanatory tools for random forests inapplicable. To disclose the impact of the original features in the deep layers, we design a calculation method with an estimation step followed by a calibration step for each layer, and propose our feature contribution and MDI feature importance calculation tools for deep forest. Experimental results on both simulated data and real world data verify the effectiveness of our methods. 
	
	\vspace{4pt}\noindent\keywords{ensemble methods, deep forest, feature importance, interpretability}
\end{abstract}

\section{Introduction}\label{sec:introduction}

By suggesting that the key to deep learning may lie in the \textit{layer-by-layer processing}, \textit{in-model feature transformation} and \textit{sufficient model complexity}, \citet{zhou17deep} propose the first deep forest model and the gcForest algorithm, which are realized by non-differentiable modules without backward-propagation in training. It consists of a cascade forest structure, each layer contains multiple random forests, and the predictive probability vectors output by the forests are concatenated with the original features, then serve as the input to the next layer~\cite{zhou2019deep}. Benefiting from the feature transformation in cascade structure, deep forests (DFs) outperform various classical tree-based algorithms, e.g., classification and regression tree \citep[CART]{breiman1984classification}, AdaBoost \citep{schapire2012boosting}, random forest \citep[RF]{breiman2001random}, gradient boost decision tree \citep[GBDT]{friedman2001greedy}, extremely random forest \citep[ERF]{geurts2006extremely} and XGBoost \citep{chen2016xgboost} in empirical studies. In recent years, deep forests have been widely extended to various real-world applications \citep{su2019deep,boualleg2019remote,zhang2019distributed} and learning tasks \citep{utkin2018siamese,utkin2019discriminative,yang2020multilabel,wang2020learning}. There are also variants aiming at improving performance and reducing computational and memory cost \citep{zhou2019deephash,pang2020improving,chen2021improving,ma2022hw}. 

Compared with deep neural networks (DNNs), the decision tree model \citep{loh2011classification}, which is the basic component of deep forest, is easier to interpret while improving the prediction performance. In order to make a prediction, trees ask each observation a series of questions, each one being the form:
\begin{equation}
	\text{IF:}\ \  x^{(j)}\gtreqless s,\ \  \text{THEN:} \ \ response\ ,
\end{equation}
where $(j,s)$ are to be determined by the splitting algorithm and $\gtreqless$ represents $\geq$ or $<$. Thus, the prediction for any instance only depends on the sequence of decision rules. However, forests prediction cannot be explained through decision rules because their prediction is the average of a large number of randomized decision trees \citep{breiman2001random,geurts2006extremely}. But researchers develop approach to compute \textit{feature contribution} for random forest regression and classification models~\citep{kuz2011interpretation,palczewska2014interpreting}. 
It enables us to explain every single prediction of a random forest through the contribution of each feature. 
While feature contribution focus on the local interpretation, Mean Decrease Impurity (MDI)~\citep{breiman2003random} offers a global \textit{feature importance} measure over the whole data set. It sums up the total reduction of the splitting criterion brought by that feature as the intrinsic feature importance measure. 

Since deep forest is more powerful than random forest, while we enjoy its good predictive performance, we naturally wants to know why it makes certain predictions. 
However, because of the feature transformation in the cascade forest structure, as \citet{lyu2022depth} have proved that the new features dominate the forests in the second and following layers, we lose the idea of how the original features function in deep layers. The interpretation tools for random forests are no longer applicable for deep forests, and knowing that feature transformations are important does not help the end user understand why the model makes specific predictions. 

To overcome this problem, we develop a two-step calculation process and extend feature contribution and feature importance to deep forest. The estimation step associates the contributions of the new features with the contributions of the original features at the previous layer through the specific training samples in the splitting nodes. The calibration step ensures that our proposed method yields proper feature contribution and feature importance. 

Our main contributions are summarized as follows. 
\begin{itemize}
	\item We propose a feature contribution calculation method for deep forest, enabling us to explain its predictions for each instance, i.e., whether the considered feature enlarges or reduces the predicted value for the considered class, and by how much. 
	\item We propose an MDI feature importance calculation method for deep forest, enabling us to tell the overall impact of each feature in building the whole model. 
	\item We analyze the properties that feature contribution and feature importance should have to ensure that our designed methods are proper feature contribution and feature importance. 
	\item Experimental results show that our proposed tools faithfully reflect the influence of each feature on deep forest prediction for both regression and classification problems.
	Furthermore, since deep forest is more powerful than random forest, our MDI feature importance also exhibits better estimation quality than that of random forest. 
\end{itemize}

The rest of this paper is organized as follows. Section~\ref{sec:relatedwork} introduces related work. Section~\ref{sec:background} briefly describes existing methods for calculating feature contribution and feature importance for random forests and discuss why they are not directly applicable to deep forest. Section~\ref{sec:properties} analyzes the properties of feature contribution and feature importance to provide guidance for developing our own methods for interpreting deep forests. Section~\ref{sec:DF-FC} and Section~\ref{sec:DF-FI} present our proposed explaining tools. Section~\ref{sec:experiments} reports the experimental results and Section~\ref{sec:conclusion} concludes the paper.

\section{Related work}\label{sec:relatedwork}

\paragraph{Deep forest}
Since \citet{zhou17deep} propose the first deep forest model, there are mainly three lines of work to study it. A line of work establishes a screening framework to reduce computational cost and memory requirement for deep forest, including confidence screening \citep{pang2018improving}, feature screening \citep{pang2020improving}, hash-based method \citep{zhou2019deephash,ma2022hw} and stability-based method \citep{chen2021improving}. 
The second line of work extends deep forest algorithms to different learning tasks and real applications. \citet{zhang2019distributed} use deep forest to achieve outstanding identification performance for financial anti-fraud system. \citet{yang2020multilabel} employ the measure-aware mechanisms to help deep forest solve multi-label problems. 
\citet{wang2020learning} address weak-label learning by using a label complement procedure in the feature transformation process of deep forest. For metric learning tasks, \citet{utkin2018siamese,utkin2019discriminative} propose a Siamese deep forest as an alternative to the Siamese neural network. 
Empirical successes have attracted a line of work to the theoretical analysis of deep forest. To start theoretical analysis of deep forests, \citet{lyu2019refined} prove a margin-based generalization bound for the additive cascade forest. For the consistency of deep forests, \citet{arnould2021analyzing} prove a tight bound on the excess risk of two-layer centered random trees, which regards the second layer as a variance reducer. Assuming that the raw and new features are separated and independently used in two stages, \citet{lyu2022region} prove that new features are easy to cause overfitting risk. 

Although deep forest has been widely used and studied, very little work has been done to help users understand the specific predictions made by deep forest models. \citet{kim2020interpretation} tried to improve the explainability by simplifying a trained deep random forest model. Although they select the most important paths in the component forests, going through all the selected paths layer by layer is still a big workload, hence the interpretability is limited. 
In this paper, we aim to explain the predictions made by the original trained deep forest, and the explanation is by directly telling how much the considered feature enlarges or reduces the predicted value for a given class. 

\paragraph{Explaining random forests}
Random forest has long been a successful machine learning method, especially for tabular data~\citep{grinsztajn2022tree}. A wealth of work has focused on explaining how each input feature functions in the model. 
To characterize the overall importance of each feature, there are two widely used measures: the Mean Decrease Impurity \citep[MDI]{breiman2003random} and the Mean Decrease Accuracy \citep[MDA]{breiman2001random}. MDI sums up the total reduction of splitting criterion caused by each feature respectively. It is known to favor features with many categories \citep{strobl2007bias,nicodemus2011stability} and may lead to systematic bias in feature selection by incorrectly assigning high importance to irrelevant features \citep{strobl2008danger,nicodemus2009predictor,li2019debiased,zhou2021unbiased}. 
On the other hand, MDA measures the importance of a feature by the reduction in the accuracy after randomly permuting the sample values of a given feature. Different permuting choices have been studied by \citet{strobl2008conditional,janitza2018computationally}. 
In empirical studies, \citet{wright2016little} show that MDI and MDA can capture interactions between features but are unable to differentiate from the marginal effects. From a theoretical perspective, there are only a few results on the feature importance of tree-based models. \citet{louppe2013understanding} investigate the theoretical MDI measure (infinite sample regime) when all features are categorical. After that, the result is extended by \citet{sutera2016context} to the context-dependent features case. 
Some additional feature importance measures, such as split count \citep{strobl2007bias,basu2018iterative} can also be used. 
Recently, there are studies focusing on interpreting every single prediction of random forest through feature contribution~\citep{palczewska2014interpreting,Saabas2014interpreting}. Meanwhile, using this tool, a debiased MDI estimate can be obtained~\cite{li2019debiased}. It functions well even when there are a large number of irrelevant features and severe noise that MDA fails because of poor accuracy. 
Since MDA is applicable to any black box models, it is also applicable to deep forest. However, to our knowledge, there is no existing work to extend feature contribution and MDI feature importance measure to deep forest. 

\section{Preliminaries}\label{sec:background}

In this section, we briefly introduce feature contribution and feature importance, and discuss why they are not directly applicable to deep forest. The key symbols and notations used in this paper are listed in Table~\ref{tab:notation}.

\begin{table}[ht]
	\centering
	\resizebox{0.9\linewidth}{!}{
		\begin{tabular}{c|c|p{13cm}}
			\toprule
			\textbf{Subject} & \textbf{Sign} & \textbf{Description} \\
			\midrule
			\multirow{4}*{Setting} & $\mathcal{D}$ & The training data.\\  
			& $n$ & The total number of training instances.\\
			& $K$ & The number of features.\\
			& $C$ & The number of classes.\\\midrule
			\multirow{10}*{\shortstack{Tree\\ \&\\ Forest}} & $t$ & A tree node.\\
			& $l$ & The depth of a leaf node. \\
			& $I(T)$ & All the internal nodes in tree $T$. \\
			& $n(t)$ & The number of training instances in node $t$.\\
			& $s(t)$ & The splitting feature of an internal node $t$. \\
			& $\Delta_{\mathcal{I}}(t)$ & The decrease of impurity by splitting node $t$.\\
			& $\mu_0$ & The average response of training data.\\
			& $\mu(t)$ & The average response of the training data in node $t$.\\
			& $\Delta \mu(t_i)$ & The change in average response of node $t_i$ and its parent node $t_{i-1}$.\\
			& $f_F(\vx)$ & The prediction of forest $F$ on $\vx$. \\
			& $f_{F,k}(\vx)$ & The amount that feature $k$ contributes to $F$'s prediction value on $\vx$. \\
			& $\operatorname{MDI}(k,F)$ & The MDI feature importance of feature $k$ in forest $F$. \\
			\midrule
			\multirow{12}*{\shortstack{Deep\\ Forest}}& $\hat{y}_i$ & The prediction of previous layer forest on $\vx_i$.\\
			& $\hat{\mu}(t)$ & The average predictive value by previous layer of the training data in node $t$, which can be viewed as average response estimated by previous layer's prediction.\\
			& $\Delta\hat{\mu}(t_i)$ & The change in average response of node $t_i$ and its parent node $t_{i-1}$ estimated by previous layer's prediction.\\
			& $\Delta\hat{\mu}(t_i,k)$ & Estimated average response change caused by feature $k$.\\
			& $\Delta\widetilde{\mu}(t_i,k)$ & The calibrated average predictive response change caused by feature $k$.\\
			& $k'$ & A new feature generated by previous layer's prediction.\\
			& $K'$ & The number of new features in the second and above layers.\\
			& $g$ & The calibration function that maps the estimated feature contribution $\left(\Delta \hat \mu(t_i,k)\right)_{k=1}^K$ to the calibrated feature contribution $\left(\Delta \widetilde \mu(t_i,k)\right)_{k=1}^K$.\\
			& $\widetilde{f}_{F', k}(\boldsymbol{x})$ & The calculated contribution of the original feature $k$ for the second layer forest $F'$.\\
			& $f_L(\vx)$ & The prediction of $\vx$ from the last layer of deep forest.\\
			& $\widetilde{f}_{L,k}(\boldsymbol{x})$ & The calculated contribution of feature $k$ in the last layer of deep forest.\\
			& $\widehat{\operatorname{MDI}}(k,DF)$ & The estimated MDI feature importance of feature $k$ in DF. \\
			\bottomrule
	\end{tabular}}
	\caption{Key symbols and notations.}
	\label{tab:notation}
\end{table}

\subsection{Feature contribution}
\label{sec:featurecontribution}
Feature contribution is applicable for every single prediction. It enables us to decompose the prediction into the sum of contributions from each feature~\citep{palczewska2014interpreting}. 

\paragraph{For decision trees}
The prediction of an instance $\boldsymbol{x}$ is the average of the training instances in the leaf node it falls in. Let $l$ denote the depth of the leaf node that $\vx$ falls in. Along the decision path to the leaf node, as $\vx$ goes through node $t_0, t_1, \ldots, t_{l}$, the average response of the training instances changes from node to node.
As demonstrated in Figure~\ref{fig:tree contribution}, the final prediction can be represented as 
\begin{equation}
	f_T(\boldsymbol{x}) = \mu(t_0)+\sum_{0 < i\leq l} \Delta \mu(t_i)\ ,
\end{equation}
where $\Delta \mu(t_i) = \mu(t_{i})-\mu(t_{i-1})$. $\Delta \mu(t_i)$ is the change in average response caused by splitting node $t_i$ from $t_{i-1}$.
Note that this calculation is applicable for both regression trees and classification trees. For regression, $f_T(\boldsymbol{x})$, $\mu(t_0)$ and $\Delta \mu(t_i)$ are scalars. For classification, they are $C$-dimensional vectors, where $C$ denotes the number of classes. 
Let $s(t)$ denote the splitting feature of node $t$. Let $K$ denote the number of input features. If we sum up the split contributions made by the same feature, the prediction of $\boldsymbol{x}$ can be written as
\begin{equation}
	f_T(\boldsymbol{x})= \mu(t_0)+\sum_{k=1}^K f_{T, k}(\boldsymbol{x})\ ,
\end{equation}
where 
\begin{equation}
	f_{T, k}(\boldsymbol{x}) = \sum_{0<i\leq l:s(t_{i-1})=k} \Delta \mu(t_i)
\end{equation}
is the contribution of feature $k$ in the prediction of tree $T$ on $\vx$.

\begin{figure}[!t]
	\centering
	\includegraphics[width=0.6\linewidth]{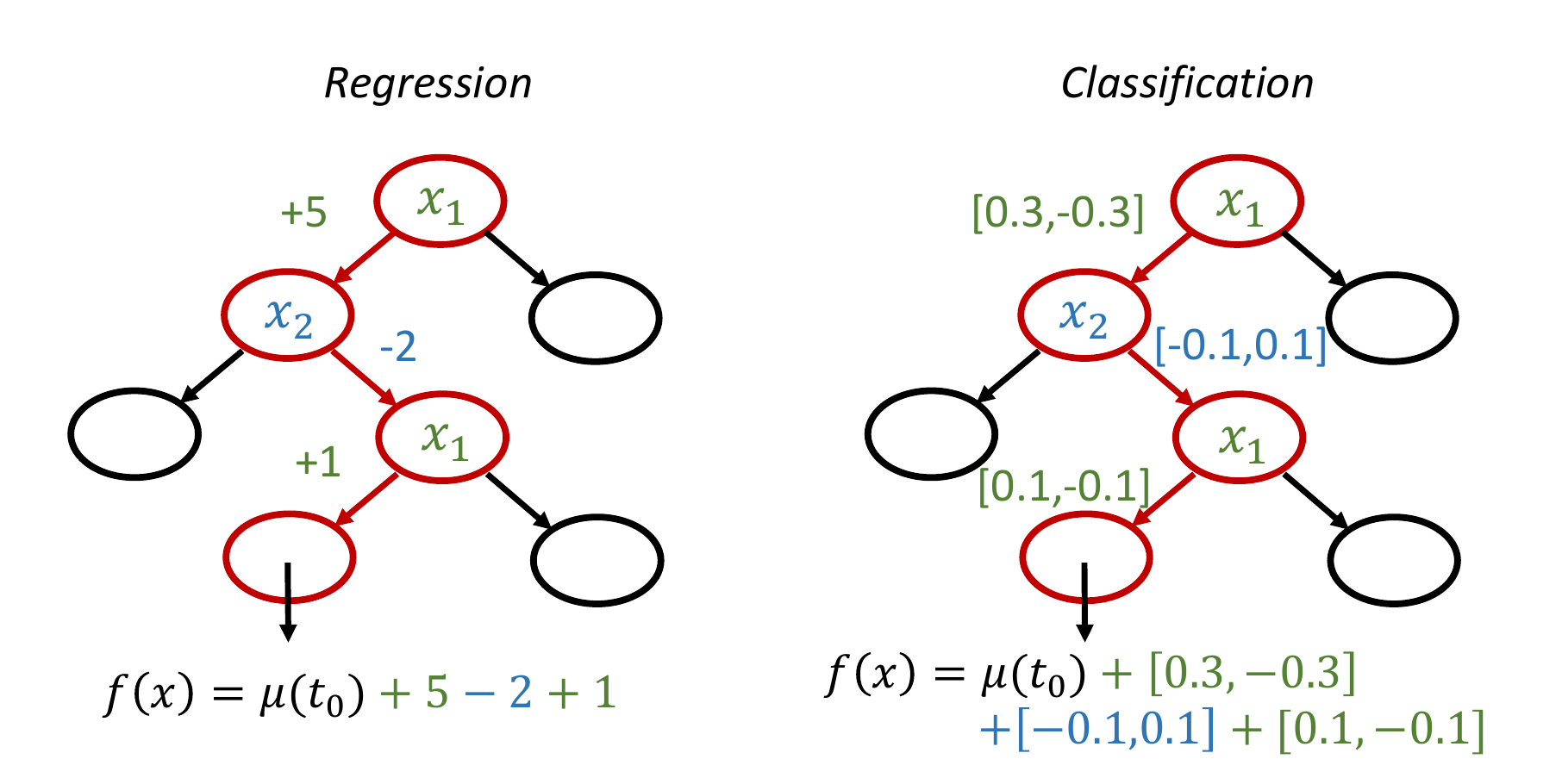}
	\caption{The contribution of each split to the prediction of an instance $\boldsymbol{x}$. Color blue and green correspond to the change of average response by splitting on $x_1$ or $x_2$. Positive and negative values indicate the split enlarges or decreases the predictive value or the probability for the corresponding class.}
	\label{fig:tree contribution}
\end{figure}

\paragraph{For random forests}
To extend feature contribution from a tree to a forest, we only need to take the average of the trees in the forest. Assuming each tree has the same distribution of training instances, we use $\mu_0$ to substitute $\mu(t_0)$. Therefore,
\begin{equation}
	\label{eq:forest}
	f_F(\boldsymbol{x})= \mu_0 + \sum_{k=1}^K f_{F, k}(\boldsymbol{x}),
\end{equation}
where
\begin{equation}
	f_{F, k}(\boldsymbol{x}) = \frac{1}{|F|}\sum_{T\in F} f_{T, k}(\boldsymbol{x})
\end{equation}
represents the contribution of feature $k$ in forest $F$'s prediction.

\subsection{Feature importance}
Mean Decrease Impurity (MDI)~\citep{breiman2003random} is an intrinsic feature importance measure for random forests. It measures the overall importance of each feature in building the whole model. 
For a tree $T$, its MDI is calculated by summing the impurity decrease in each internal node weighted by the fraction of training data in the node. Let $n$ denote the number of training instances, $I(T)$ denote all the internal nodes in tree $T$, $n(t)$ denote the number of training instances falling in node $t$. The feature importance of feature $k$ in tree $T$ is calculated as
\begin{equation}
	\label{eq:MDI}
	\operatorname{MDI}(k, T)=\sum_{t \in I(T), s(t)=k} \frac{n(t)}{n} \Delta_{\mathcal{I}}(t)\ ,
\end{equation}
where the decrease impurity of any node $t$ is defined as

\begin{equation}
	\begin{split}
		\Delta_{\mathcal{I}}(t)\triangleq&\operatorname{Impurity}(t)-\frac{n\left(t^{\text {left}}\right)}{n(t)} \operatorname{Impurity}\left(t^{\text {left}}\right)
		-\frac{n\left(t^{\text {right}}\right)}{n(t)} \operatorname{Impurity}\left(t^{\text {right}}\right)\ ,
	\end{split}
\end{equation}
and the impurity of any node $t$ is defined as
\begin{equation}
	\label{eq:impurity}
	\operatorname{Impurity}(t)\triangleq\frac{1}{n(t)} \sum_{i: \mathbf{x}_{i} \in t}\left(y_{i}-\mu(t)\right)^{2}\ .
\end{equation}
Here $t^{\text {left}}$ and $t^{\text {right}}$ are two child nodes of node $t$, $\mu(t)$ is the average response of training instances in node $t$.
Note that the above expressions assume using the sample variance of labels as the impurity measure for regression problems. For classification problems, \citet{li2019debiased} disclose that using one-hot encoding of the categorical responses, i.e., substituting $y$ and $\mu(t)$ with vectors, the impurity expression in Eq.~\eqref{eq:impurity} is equivalent to Gini index, a popular impurity measure for classification.

Since the forest is an average of all the individual trees, the feature importance of feature $k$ in forest $F$ is also the average 
\begin{equation}
	\operatorname{MDI}(k, F)=\frac{1}{|F|}\sum_{T\in F}\operatorname{MDI}(k, T)\ .
\end{equation}

\subsection{Deep forest and the difficulty of interpretation}

Figure~\ref{fig:DF} illustrates the cascade forest structure of deep forests. Each layer consists of multiple random forests, with different colors indicating different kinds of random forests. The whole model is built of multiple layers of forests, and the number of layers is automatically determined by the validation accuracy. Note that each layer outputs its prediction vectors (in gray color) to serve as new features for the next layer. That is to say, for the second and deeper layers, the input features are the concatenation of the original features (in red color) and new features. With the layer-by-layer processing, the predictive performance can be improved accordingly.

\begin{figure*}[t]
	\subfigure[The cascade forest structure of deep forests. Color black and blue indicate different kinds of random forests. The vectors in gray color are the new features generated by individual forests, being input to the next layer.\label{fig:DF}]{
		\begin{minipage}[h]{0.53\columnwidth}
			\centering
			\includegraphics[height=0.20\textheight]{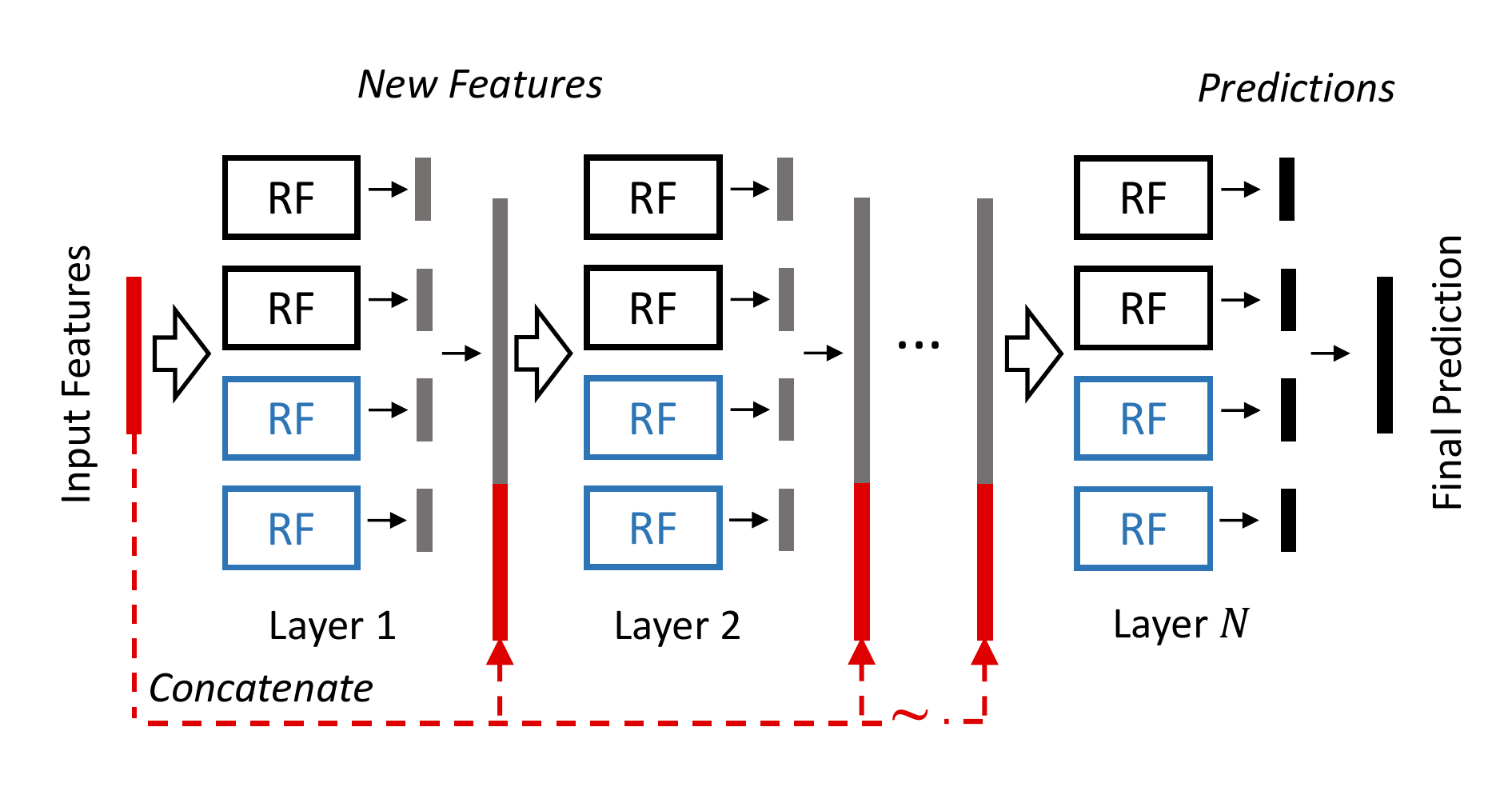}
	\end{minipage}}
	\hfill
	\subfigure[A typical MDI feature importance result of the second and above layers. $X_{new}$ is the sum of MDI of new features. The new features tend to have dominating importance.\label{fig:histImportance}]{
		\begin{minipage}[h]{0.43\columnwidth}
			\centering
			\includegraphics[height=0.20\textheight]{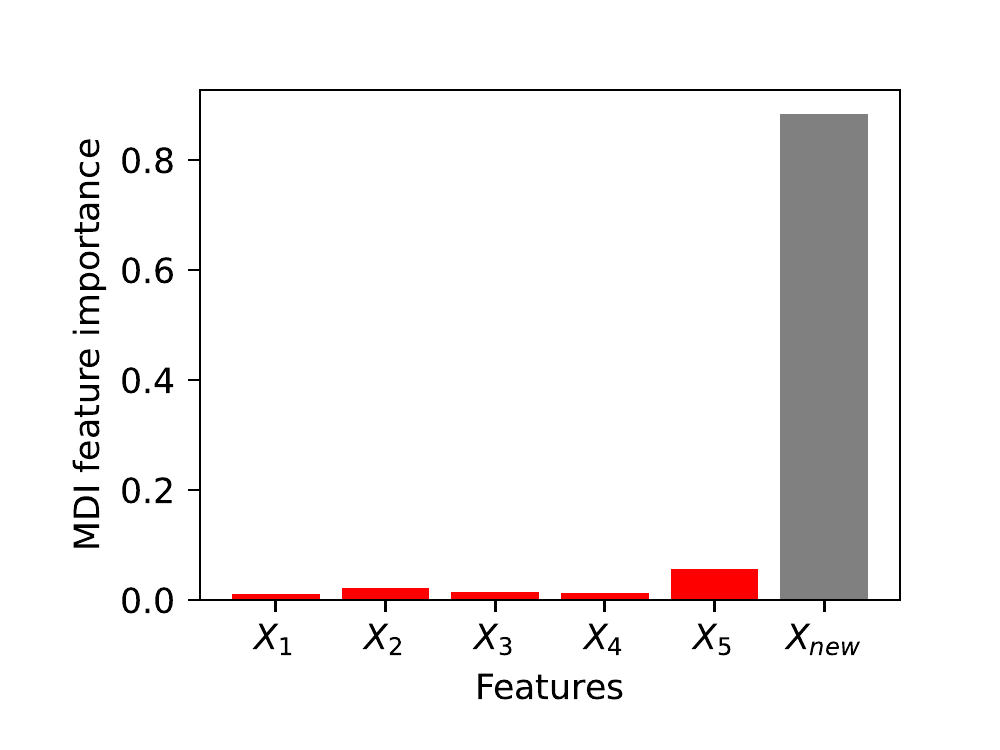}
	\end{minipage}}
	\caption{The new features play an important role in deep forests but are hard to interpret.}
	\label{fig:features}
\end{figure*}

Existing empirical results and theoretical analyses~\citep{arnould2021analyzing,lyu2022region} show that in the second and following layers, the individual trees in random forests will primarily split on the new features, especially in the first several nodes of decision trees. For an arbitrary data set, the feature importance in the second layer forest will very likely behave like Figure~\ref{fig:histImportance}. Apart from the fact that the new features are very important in prediction, it does not help us to understand the prediction made by deep forest because the new feature is the interaction of all the old features. 

\section{Properties of feature importance and feature contribution}
\label{sec:properties}
To interpret deep forests, we start with analyzing the properties that feature contribution and feature importance should have and their relationship. 
These serve as guidance for developing interpretation tools for deep forests, ensuring that our designed methods are proper feature importance and feature contribution.

\subsection{Properties of feature contribution}
The term \emph{feature contribution} means decomposing the prediction on a single instance into a sum of contributions from each feature.
A proper feature contribution should satisfy the following properties,
\begin{itemize}
	\item Feature contribution is defined for each single prediction.
	\item 
	In regression problems, feature contribution is a $K$-dimensional vector, whose element $f_{k}(\boldsymbol{x})$ denotes the contribution of feature $k$ in the final prediction. In classification problems, 
	feature contribution is a matrix of shape $(K,C)$, its element $f_{k,c}(\boldsymbol{x})$ denoting the contribution of feature $k$ in the predictive probability of class $c$.
	\item The elements of feature contribution can be positive or negative or zero. 
	\item Given an instance $\vx$, the sum of bias and feature contribution equals the prediction. That is, 
	for regression,
	\begin{equation}
		f(\boldsymbol{x})= \mu_0+\sum_{k=1}^K f_{k}(\boldsymbol{x})\ .
	\end{equation}
	For classification,
	\begin{equation}
		f(\boldsymbol{x})_c= \mu_{0,c}+\sum_{k=1}^K f_{k,c}(\boldsymbol{x}), 1\leq c\leq C\ .
	\end{equation}
\end{itemize}

\subsection{Properties of feature importance}
The term \emph{feature importance} means the overall impact of each feature in building the whole model. 
A proper feature importance measure should have the following properties,
\begin{itemize}
	\item Feature importance is defined over a whole data set, with labels available.
	\item Feature importance is a $K$-dimensional vector, for both regression and classification.
	\item The elements of feature importance are non-negative. 
	\item The sum of feature importance is no larger than the total response variance of the data set. For regression, the total response variance is $\frac{1}{n} \sum_{i} (y_i - \mu_0)^2$. For classification, the total response variance is defined to be $\frac{1}{n} \sum_i \sum_{1\leq c\leq C}(y_{i,c}-\mu_{0,c})^2$.
\end{itemize}

Note that here we use variance-based definitions for feature importance. We can easily scale the feature importance to sum to 1 by normalizing with the total response variance.

\subsection{Relationship between feature contribution and feature importance}
Recently, \citet{li2019debiased} disclose that the original expression of MDI as Eq.~\eqref{eq:MDI} can be written as
\begin{equation}
	\label{eq:covrelationship}
	\operatorname{MDI}(k, T) 
	= \frac{1}{\left|\mathcal{D}\right|} \sum_{i \in \mathcal{D}} f_{T, k}\left(\boldsymbol{x}_{i}\right) \cdot y_{i}
	= \operatorname{Cov}\left(f_{T,k}\left(\boldsymbol{x}_{i}\right), y_{i}\right).
\end{equation}

Here $\mathcal{D}$ is the data set we use for computing MDI feature importance. If $\mathcal{D}$ is the same as the training set, then Eq.~\eqref{eq:covrelationship} yields the same result as the original MDI expression. \citet{li2019debiased} claim that using out-of-bag data to calculate Eq.~\eqref{eq:covrelationship} can achieve an unbiased result. 

Equation~\eqref{eq:covrelationship} offers a new way of computing MDI feature importance in a decision tree $T$, i.e., the covariance of feature contribution $f_{T,k}(\boldsymbol{x})$ and label $y$. Before, we could only compute MDI by summing the decrease of impurity at each node. Equation~\eqref{eq:covrelationship} enables us to compute feature importance by summing over instances, decoupling the computation of feature contribution from the tree training process.

For a forest $F$, suppose we use the same data set $\mathcal{D}$ to calculate MDI for every tree in it, its MDI can be calculated as
\begin{equation}
	\label{eq:MDI-forest}
	\operatorname{MDI}(k, F) 
	=\frac{1}{|F|}\sum_{T \in F} \operatorname{MDI}(k, T) 
	= \frac{1}{\left|\mathcal{D}\right|} \sum_{i \in \mathcal{D}} y_{i} f_{F, k}\left(\boldsymbol{x}_{i}\right).
\end{equation}

\section{Feature contribution for deep forest}
\label{sec:DF-FC}

To compute feature contribution for deep forest, the main difficulty lies in the splits on the new features. We propose to use the training samples in the corresponding node to link the influence of splits back to the last layer forests. 
However, the estimated contribution of the original features through last layer's prediction are probably imprecise, and the error will accumulate layer by layer. Therefore, we further propose a calibration step to ensure that the calculation result after each layer maintains a proper feature contribution measure.
\subsection{The estimation step}
Suppose the node $t_i$ in a second-layer tree $T'$ uses a new feature $k'$ to split. We can estimate the average response using the predictions from the previous layer
\begin{align*}
	\hat{\mu}(t_i)  &= \frac{1}{n(t_i)} \sum_{j: \boldsymbol{x}_{j} \in t_i} \hat{y}_{j}\\
	&= \frac{1}{n(t_i)} \sum_{j: \boldsymbol{x}_{j} \in t_i} \left(\mu_0 + \sum_{k=1}^K f_{F, k}(\boldsymbol{x}_j)\right)\\
	&= \mu_0 + \frac{1}{n(t_i)} \sum_{j: \boldsymbol{x}_{j} \in t_i}\left(\sum_{k=1}^K f_{F, k}(\boldsymbol{x}_j)\right),
\end{align*}
where $F$ is the forest in the previous layer that generates $k'$ (i.e., $\hat{y}$), and $f_{F,k}(\vx)$ is the amount that feature $k$ contributes to $F$'s prediction value on $\vx$.

Thus the estimated response change from splitting cell $t_{i-1}$ into $t_i$ is
\begin{align}
	\Delta \hat{\mu}(t_i)=&
	\frac{1}{n(t_{i})} \sum_{j: \boldsymbol{x}_{j} \in t_{i}}\left(\sum_{k=1}^K f_{F, k}(\boldsymbol{x}_j)\right)
	-\frac{1}{n(t_{i-1})} \sum_{j: \boldsymbol{x}_{j} \in t_{i-1}}\left(\sum_{k=1}^K f_{F, k}(\boldsymbol{x}_j)\right)\ .
\end{align}

Let
\begin{align}
	\label{eq:estimated}
	\Delta \hat\mu(t_i,k)=&\frac{1}{n(t_{i})} \sum_{j: \boldsymbol{x}_{j} \in t_{i}} f_{F, k}(\boldsymbol{x}_j)
	-\frac{1}{n(t_{i-1})} \sum_{j: \boldsymbol{x}_{j} \in t_{i-1}} f_{F, k}(\boldsymbol{x}_j)
\end{align}
denote the contribution of feature $k$ in the estimated change of average response, then
\begin{equation}
	\Delta \hat{\mu}(t_i)=\sum_{k=1}^K{\Delta \hat\mu(t_i,k)}\ .
\end{equation}

Therefore we attribute the estimated response change caused by splitting node $t_i$ using feature $k'$ to the contribution of original features.

\subsection{The calibration step}\label{sec:calibration}

Meanwhile, using the labels of data, the actual change in average response from node $t_i$ to $t_{i+1}$ is calculated as 
\begin{align}
	\label{eq:directdelta}
	\Delta \mu(t_i)=&
	\mu(t_{i})-\mu(t_{i-1})
	=\frac{1}{n(t_{i})} \sum_{j: \boldsymbol{x}_{j} \in t_{i}} {y}_{j}-\frac{1}{n(t_{i-1})} \sum_{j: \boldsymbol{x}_{j} \in {t_{i-1}}} {y}_{j}\ .
\end{align}

However, $\Delta \hat{\mu}(t_i)$ and $\Delta \mu(t_i)$ are usually unequal, because the predicted response of $\vx_j$ is usually not exactly equal to its true label. And the error might propagate layer by layer. To avoid this, we propose a calibration step after estimation to ensure the computed feature contribution satisfies the property that the sum of bias and feature contribution equals the prediction as stated in Section \ref{sec:properties}. 

The calibration operation can be expressed as a function
\begin{align}
	\label{eq:calibration}
	g: &\left(\Delta \hat\mu(t_i, 1),\Delta \hat\mu(t_i, 2),\ldots, \Delta \hat\mu(t_i, K)\right)
	\rightarrow \left(\Delta \widetilde\mu(t_i, 1),\Delta \widetilde\mu(t_i, 2),\ldots, \Delta \widetilde\mu(t_i, K)\right)\ ,
\end{align}
such that
\begin{align}\label{eq:checkcalibration}
	\sum_{k=1}^K{\Delta \widetilde\mu(t_i, k)} = \Delta \mu(t_i)\ .
\end{align}

Thus the prediction of $f_{T'}(\boldsymbol{x})$ can be interpreted as the following decomposition, 
\begin{equation}
	f_{T'}(\boldsymbol{x})= \mu_0 +\sum_{k=1}^K \widetilde{f}_{T', k}(\boldsymbol{x})\ ,
\end{equation}
where
\begin{equation}
	\widetilde{f}_{T', k}(\boldsymbol{x}) = \sum_{s(t_{i-1})=k} \Delta \mu(t_i) + \sum_{s(t_{i-1})=k'}\Delta \widetilde \mu(t_i,k)
\end{equation}
is the calibrated contribution of the $k$-th original feature.

\paragraph{Naive multiplicative calibration}
A naive instantiation of Eq.~\eqref{eq:calibration} is to re-scale the decomposition using a common scale factor, 
\begin{equation}
	\label{eq:multiplicative}
	\begin{split}
		\Delta \widetilde\mu(t_i,k) =  \Delta \hat\mu(t_i,k)\frac{\Delta \mu(t_i)}{\Delta \hat{\mu}(t_i)}\ .
	\end{split}
\end{equation}
Obviously, Eq.~\eqref{eq:checkcalibration} holds. 
However, since the feature contributions can be both positive and negative, the naive re-scaling might lead to numerical problems.

\paragraph{Naive additive calibration}
Using additive calibration can better avoid numerical problems than multiplicative scaling. The naive modification is proportional to the absolute value of its estimated contribution.
\begin{equation}
	\label{eq:additive}
	\Delta \widetilde\mu(t_i,k) = \Delta \hat\mu(t_i,k)+ \frac{|\Delta \hat\mu(t_i,k))|}{\sum_{k}{|\Delta \hat\mu(t_i,k))|}}(\Delta \mu(t_i)-\Delta\hat\mu(t_i))\ .
\end{equation}
However, if the value of $(\Delta \mu(t_i)-\Delta\hat\mu(t_i))$ is large, features whose estimated contribution is negative will be modified to be positive, and the more negative it was, the more positive it will be, which is obviously not sensible.

\paragraph{Partial additive calibration}
Here we advocate a partial calibration method. 
We choose a subset of features according to their signs
\begin{equation}
	S(t_i) = \{k:\operatorname{sign}(\Delta \hat\mu(t_i,k))=\operatorname{sign}(\Delta \mu(t_i))\}\ ,
\end{equation}
and we only calibrate this subset of features
\begin{equation}\label{eq:partial}
	\Delta \widetilde\mu(t_i,k) = \Delta \hat\mu(t_i,k)\left(1+\frac{\Delta \mu(t_i)-\Delta\hat\mu(t_i)}{\sum_{k\in S(t_i)}{\Delta \hat\mu(t_i,k))}}\right)\ , \ k\in S(t_i)\ .
\end{equation}
It is easy to check that Eq.~\eqref{eq:checkcalibration} holds. Therefore, the response change brought by splitting on $k'$ has been decomposed into contributions of the original $K$ features. 

\subsection{The calculation procedure}
Algorithm~\ref{alg:calculateContribution} illustrates the procedure of calculating feature contribution for a tree.
When traversing the tree, if the node chooses an original feature to split, we simply add the contribution to this feature; if the node uses a new feature to split, we trace back the contribution to each of the corresponding original features as in Eq.~\eqref{eq:estimated} and calibrate as in Eq.~\eqref{eq:calibration}.
The calculation is done by recursively call \textit{CalculateContribution} and get the contributions at all the leaf nodes. To compute feature contribution on a given instance, we simply find the leaf node $t$ that $\vx$ falls in and get $\widetilde{f}_{T',k}(\vx) = f_{t,k}$.

\begin{algorithm} 
	\caption{Calculate feature contribution for a tree} 
	\label{alg:calculateContribution} 
	\begin{algorithmic}[1]
		\REQUIRE Tree node $t$ 
		\ENSURE CalculateContribution($t$)
		\IF{$t$ is a root node} 
		\STATE $f_{t,k}\gets 0, k=1,\ldots,K$ 
		\ENDIF 
		\FOR{each child node $t^*$ of $t$}
		\IF{$s(t)\in \{1,\ldots,K\}$} 
		\STATE Calculate $\Delta \mu(t^*,s(t))$ as in Eq.~\eqref{eq:directdelta}
		\STATE $f_{t^*,s(t)}\gets f_{t,s(t)}+\Delta {\mu}(t^*,s(t))$
		\ELSE
		\STATE Calculate $\left(\Delta \widetilde{\mu}(t^*,k)\right)_{k=1}^K$ using Eqs.~\eqref{eq:estimated} and~\eqref{eq:calibration}
		\STATE $f_{t^*,k}\gets f_{t,k}+\Delta \widetilde{\mu}(t^*,k),\ k=1,\ldots,K$
		\ENDIF
		\IF {node $t^*$ is a leaf node}
		\RETURN $\left(f_{t^*,k}\right)_{k=1}^K$
		\ELSE
		\STATE CalculateContribution($t^*$)
		\ENDIF
		\ENDFOR
	\end{algorithmic} 
\end{algorithm}

Based on the analysis of a tree, it is easy to interpret the prediction of a forest
\begin{equation}
	f_{F'}(\boldsymbol{x})= \mu_0 +\sum_{k=1}^K \widetilde{f}_{F', k}(\boldsymbol{x})\ ,
\end{equation}
where
\begin{equation}
	\widetilde{f}_{F', k}(\boldsymbol{x}) = \frac{1}{|F'|}\sum_{T'\in F'} \widetilde{f}_{T', k}(\boldsymbol{x})\ .
\end{equation}

For the above analysis, we assume $T'$ is the individual tree in the second-layer forest $F'$. But it is easy to see that, after the above calculation, we will get the contributions of the $K$ original features $f_{F',k}(\boldsymbol{x})$. Therefore, from the third layer and beyond, we can repeat the same calculation strategy.

Usually, each layer in a deep forest contains several forests. The final prediction is made by taking the average of all the outputs of the last layer of forests. Let $L$ denote the set containing forests in the last layer, the final prediction $f(\boldsymbol{x})$ made by deep forest can be interpreted as
\begin{equation}
	f_L(\boldsymbol{x})=\mu_0 +\sum_{k=1}^K \widetilde{f}_{L,k}(\boldsymbol{x}),
\end{equation}
where
\begin{equation}\label{eq:L-contribution}
	\widetilde{f}_{L,k}(\boldsymbol{x}) = \frac{1}{|L|} \sum_{F\in L} \widetilde{f}_{F,k}(\boldsymbol{x})
\end{equation}
is contribution of feature $k$ in the last layer of forests, which is also the feature contribution for the whole deep forest.

\noindent\textbf{Time complexity.}\quad
As shown in Algorithm~\ref{alg:calculateContribution}, we only need to traverse each tree once and do simple computation concerning the feature contributions of the training instances falling in each node. The depth of a tree is approximately $O(\log n)$. For nodes that have the same depth, all the training instances are processed once. 
Therefore, 
the time complexity of the computation for a tree is $O(n\log n)$. Let $M_1$ denote the number of trees in each layer, $M_2$ denote the maximum number of layers, the time complexity for computing feature contribution for deep forest is $O(M_1 M_2 n\log n)$.

\noindent\textbf{Space complexity.}\quad When computing feature contribution for a tree, we need to maintain $\left[f_{t,k}\right]_{k=1}^K$ through the traversal of a tree and store $\left[f_{t,k}\right]_{k=1}^K$ for every leaf node. $f_{t,k}$ is a $C$-dimensional vector for classification problems (set $C=1$ for regression problems), and the number of leafs node is at most $n$. Therefore, the space complexity is $O(n KC)$. 

While computing feature contribution for a forest, we only need to keep the average of tree contributions on the training instances instead of the nodes of all the trees. Therefore the space complexity is still $O(nKC)$.

While computing feature contribution layer by layer through the cascade forest structure, we only need to keep the feature contribution of the previous layer. Let $|L|$ denote the number of forests in each layer, the space complexity is $O(nKC|L|)$.

\section{Feature importance for deep forests}\label{sec:DF-FI}
After obtaining feature contribution for deep forest, according to Eq.~\eqref{eq:MDI-forest}, we can compute MDI feature contribution for deep forest (DF) as
\begin{equation}
	\label{eq:DF-FI}
	\widehat{\operatorname{MDI}}(k,DF) 
	= \frac{1}{\left|\mathcal{D}\right|} \sum_{i \in \mathcal{D}} \widetilde{f}_{L, k}\left(\boldsymbol{x}_{i}\right) \cdot y_{i},
\end{equation}
where $\widetilde{f}_{L, k}$ is defined in Eq.~\eqref{eq:L-contribution}. We then show that it is a proper feature importance measure. 

\begin{proposition}\label{prop:MDI_DF}
	The sum of the estimated MDI feature importance as Eq.~\eqref{eq:DF-FI} over all the original features equals the sum of the last layer MDI directly over the $K+K'$ features ($K'$ denoting the number of new features), i.e.,
	\begin{equation}
		\sum_{k=1}^K \widehat{\operatorname{MDI}}(k,DF) = \frac{1}{|L|}\sum_{F \in L}\sum_{k=1}^{K+K'} \operatorname{MDI}(k,F).
	\end{equation}
\end{proposition}

\begin{proof}
	For a last-layer forest $F$,
	\begin{align}
		\notag\sum_{k=1}^K\widehat{\operatorname{MDI}}(k,F)
		&= \sum_{k=1}^K\frac{1}{\left|\mathcal{D}\right|} \sum_{i \in \mathcal{D}} \widetilde{f}_{F, k}\left(\boldsymbol{x}_{i}\right) \cdot y_{i}\\
		\notag= &\sum_{k=1}^K\frac{1}{\left|\mathcal{D}\right|} \sum_{i \in \mathcal{D}}\left(\frac{1}{|F|}\sum_{T\in F} \widetilde{f}_{T, k}(\boldsymbol{x}_i)\right)\cdot y_{i}\\
		\notag= &\sum_{k=1}^K\frac{1}{\left|\mathcal{D}\right|} \sum_{i \in \mathcal{D}}\left(\frac{1}{|F|} \sum_{T\in F}\left(\sum_{j\in T:s(t_{j-1})=k} \Delta \mu(t_j) + \sum_{k'=K+1}^{K+K'}\sum_{j\in T:s(t_{j-1})=k'} \Delta \widetilde\mu(t_j,k)\right)\right)\cdot y_{i}\\ 
		\notag= &\sum_{k=1}^K\frac{1}{\left|\mathcal{D}\right|} \sum_{i \in \mathcal{D}}\left(\frac{1}{|F|}\sum_{T\in F} \sum_{j\in T:s(t_{j-1})=k} \Delta \mu(t_j) \right)\cdot y_{i} + \frac{1}{\left|\mathcal{D}\right|} \sum_{i \in \mathcal{D}}\left(\frac{1}{|F|}\sum_{T\in F} \sum_{k'=K+1}^{K+K'} \sum_{j\in T:s(t_{j-1})=k'}\Delta \mu(t_j)\right)\cdot y_{i}\\
		\notag= &\sum_{k=1}^K\frac{1}{\left|\mathcal{D}\right|} \sum_{i \in \mathcal{D}}\left(\frac{1}{|F|}\sum_{T\in F} f_{T,k}(\boldsymbol{x}_i) \right)\cdot y_{i} + \sum_{k'=K+1}^{K+K'} \frac{1}{\left|\mathcal{D}\right|} \sum_{i \in \mathcal{D}}\left(\frac{1}{|F|}\sum_{T\in F} f_{T,k'}(\boldsymbol{x}_i) \right)\cdot y_{i}\\ 
		\notag= &\sum_{k=1}^K \operatorname{MDI}(k,F)+\sum_{k'=K+1}^{K+K'} \operatorname{MDI}(k',F)\\ 
		= &\sum_{k=1}^{K+K'} \operatorname{MDI}(k,F)\ .
	\end{align}
	
	Taking the averaging over all forests in the last layer,
	\begin{align}
		\notag\sum_{k=1}^K\widehat{\operatorname{MDI}}(k,DF) &= \frac{1}{|L|}\sum_{F \in L}\sum_{k=1}^K\widehat{\operatorname{MDI}}(k,F)  
		=\frac{1}{|L|}\sum_{F \in L}\sum_{k=1}^{K+K'}\operatorname{MDI}(k,F)\ .
	\end{align}
\end{proof}

Proposition~\ref{prop:MDI_DF} shows that the proposed method is a way of disassembling the importance of the new feature to the original features, thus meeting the properties suggested in Section~\ref{sec:properties} that a proper MDI feature importance should have. 

\section{Experiments}\label{sec:experiments}
We first show that our computation methods yield reasonable results for feature contribution and feature importance of deep forest in Section~\ref{sec:exp-FI} and Section~\ref{sec:exp-MDI}. We generate simulated data sets with clear mechanisms behind them, and we conduct experiments on both regression and classification cases. 
Then in Section~\ref{sec:exp-rank}, we compare the quality of our deep forest MDI to other feature importance measures based on its ability of identifying relevant features, showing that as deep forest is more powerful than random forest, it also has better estimation of feature importance. 
In Section~\ref{sec:exp-calib}, we further compare the quality of our deep forest MDI method equipped with different calibration methods, showing that partial additive calibration is the most suitable choice.
In order to facilitate the understanding of possible applications in real-world tasks, in Section~\ref{sec:bikesharing}, a real-world bike sharing data set is taken as an example to show the calculation results of feature contribution and feature importance.

\subsection{Illustration of feature contribution for deep forest}\label{sec:exp-FI}
In this section, we illustrate our computation of feature contribution for deep forest. 
We generate synthetic data sets for regression and classification respectively. 
With the synthetic data sets, we are able to check whether the calculated feature contributions match the underlying data generating process. 
We first report the change in deep forest's test error layer by layer. Then visualize the feature contributions in the first layer and last layer respectively. 

In this experiment, each layer of deep forest has 4 forests, and each forest contains 50 trees, with their maximum depth fixed to $5$. The number of cascade layers of deep forest is determined automatically by the performance on the validation set. 

\begin{figure*}[!t]
	\centering
	\subfigure[Data generating function. \label{fig:regdata}]{
		\begin{minipage}[h]{0.47\columnwidth}
			\centering
			\includegraphics[height=0.19\textheight]{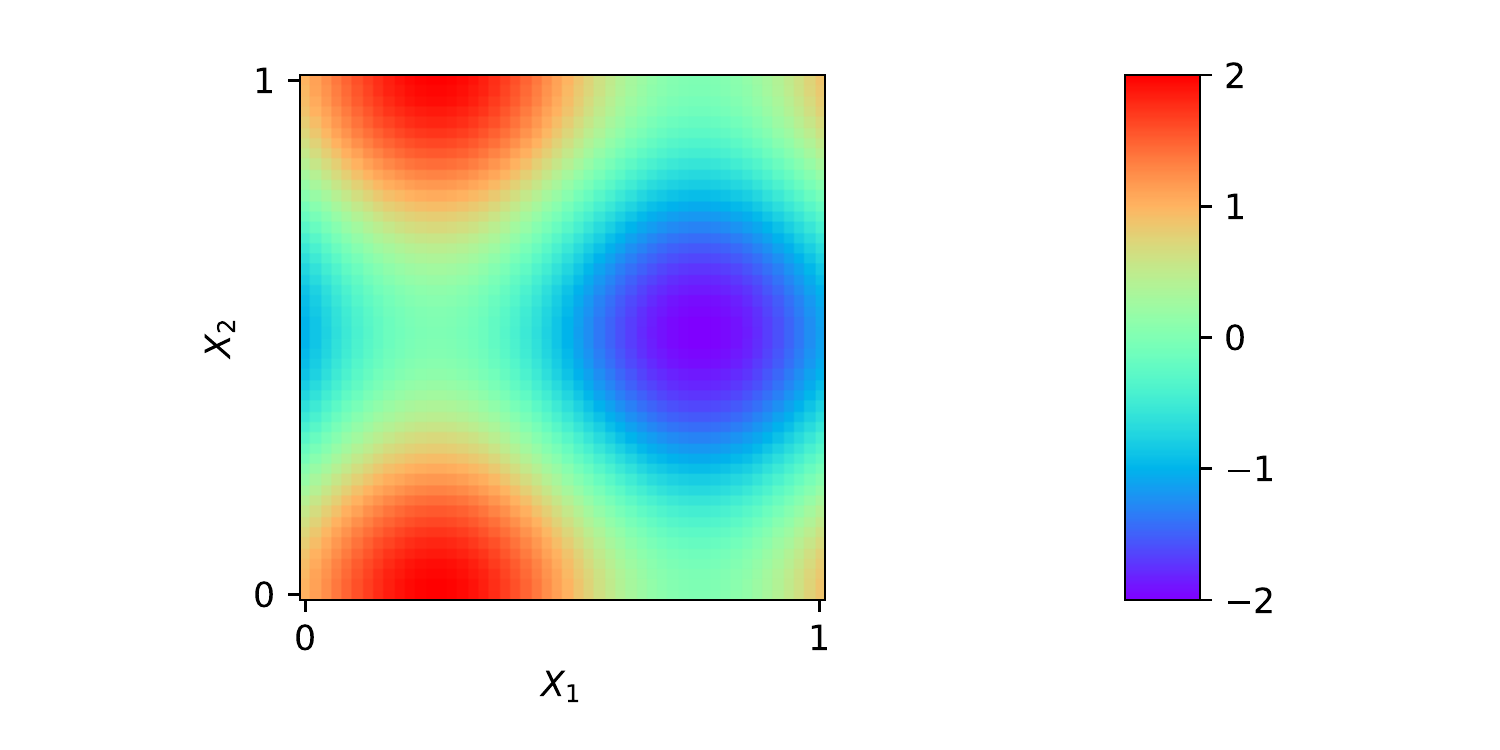}
	\end{minipage}}
	\hfill
	\subfigure[Decreasing validation error in the training process. \label{fig:regerr}]{
		\begin{minipage}[h]{0.47\columnwidth}
			\centering
			\includegraphics[height=0.19\textheight]{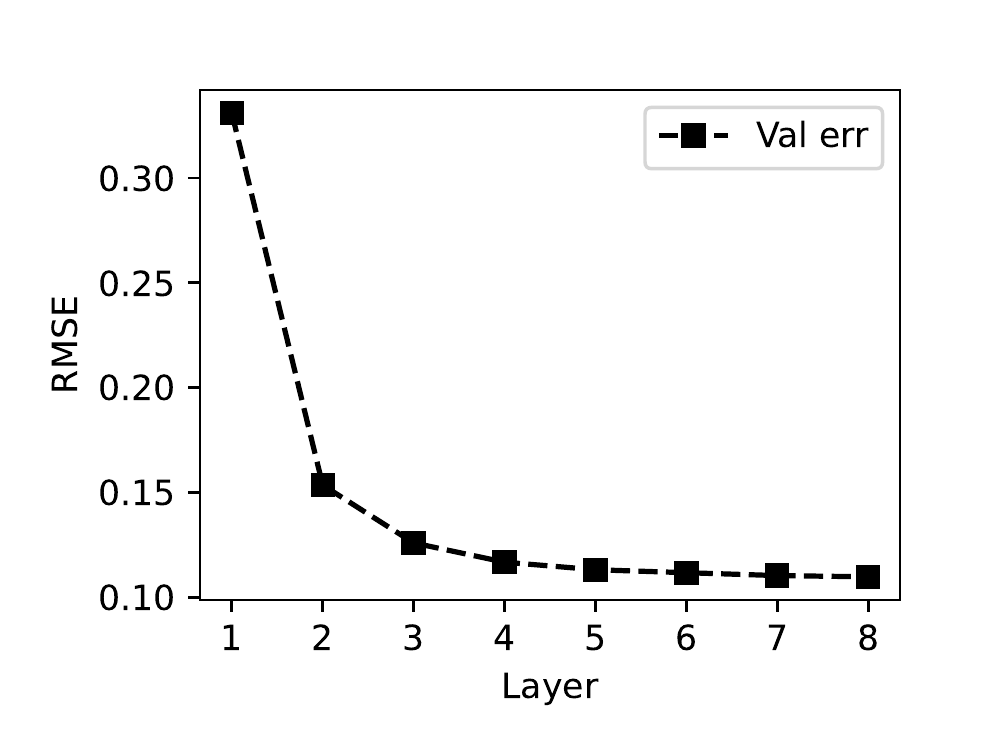}
	\end{minipage}}
	\vskip -0.1 in
	\subfigure[Prediction and feature contributions in the first layer. \label{fig:regfirstlayercontrib}]{
		\begin{minipage}[h]{0.47\columnwidth}
			\centering
			\includegraphics[width=\columnwidth]{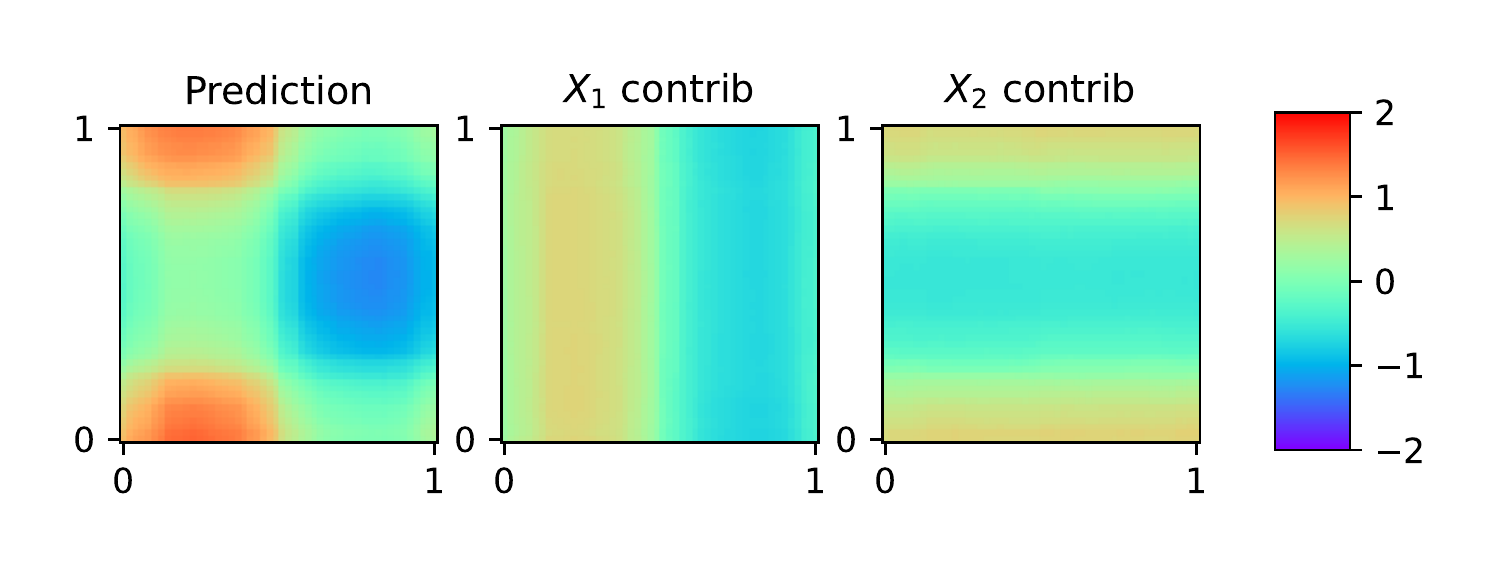}
	\end{minipage}}
	\subfigure[Prediction and feature contributions in the last layer. \label{fig:regbestlayercontrib}]{
		\begin{minipage}[h]{0.47\columnwidth}
			\centering
			\includegraphics[width=\columnwidth]{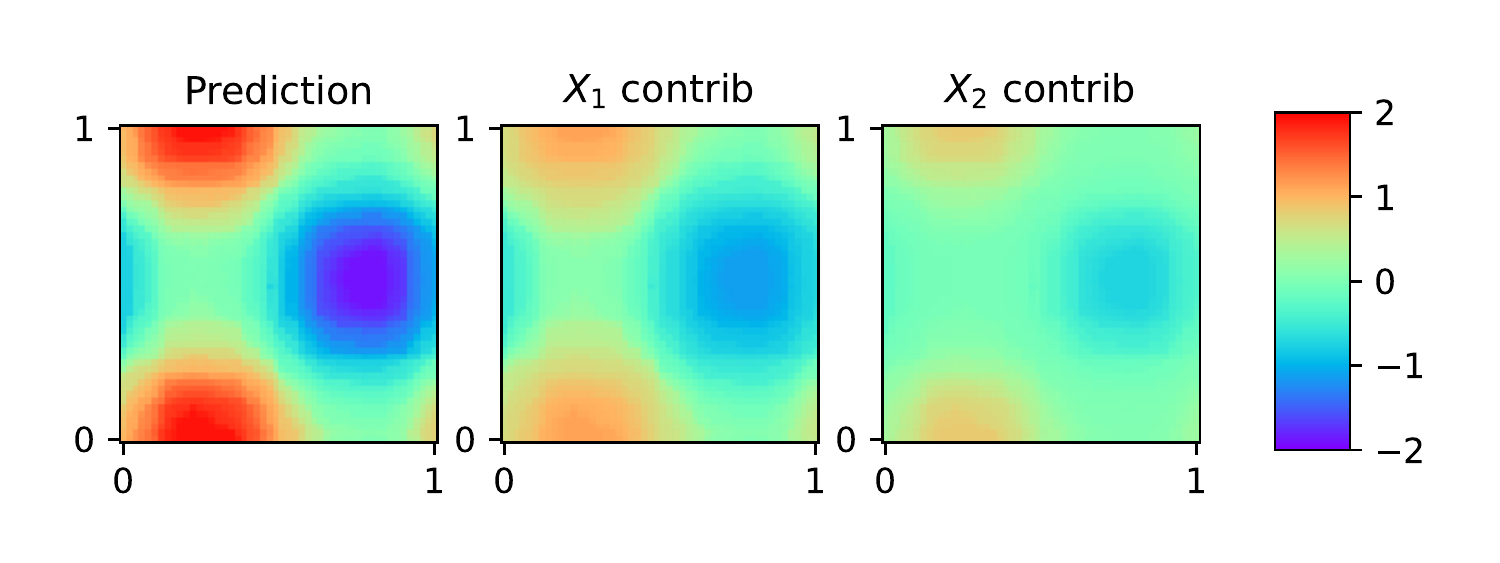}
	\end{minipage}}
	\caption{Feature contributions for the regression problem. As shown in (b) and the prediction figures in (c) (d), the predictive performance improves from the first layer to the last layer. The feature contributions show that the model captures more details in both input features.}
	\label{fig:reghot}
\end{figure*}

\paragraph{Regression}
2000 training samples and 2500 test samples are generated using function
\begin{equation}
	f(\vx) = \sin 2\pi x_1 + \cos 2\pi x_2
\end{equation}
and Figure~\ref{fig:regdata} is a heatmap of it. The decreasing validation error in Figure~\ref{fig:regerr} shows improvement of performance layer by layer. Figure~\ref{fig:regfirstlayercontrib} and \ref{fig:regbestlayercontrib} decompose the predictions in the first layer and the last layer to the contribution plots of two input features respectively. 
It is easy to observe that the first layer's prediction lacks details in the peaks and valley while last layer's prediction is much closer to the data generating function. We can also observe corresponding refinements of feature contributions. 
Note that we can see feature interaction in deep layers. In Figure~\ref{fig:regfirstlayercontrib} the change of feature contribution only along the considered feature. But in Figure~\ref{fig:regbestlayercontrib}, 
we can observe the influence of the other feature. Though this may not honestly recovered the data generating process, it helps deep forest boost performance. 

\begin{figure*}[!t]
	\vskip -0.1in
	\centering
	\subfigure[Training data shown in first two dimensions. \label{fig:3classdata}]{
		\begin{minipage}[h]{0.47\columnwidth}
			\centering
			\includegraphics[height=0.21\textheight]{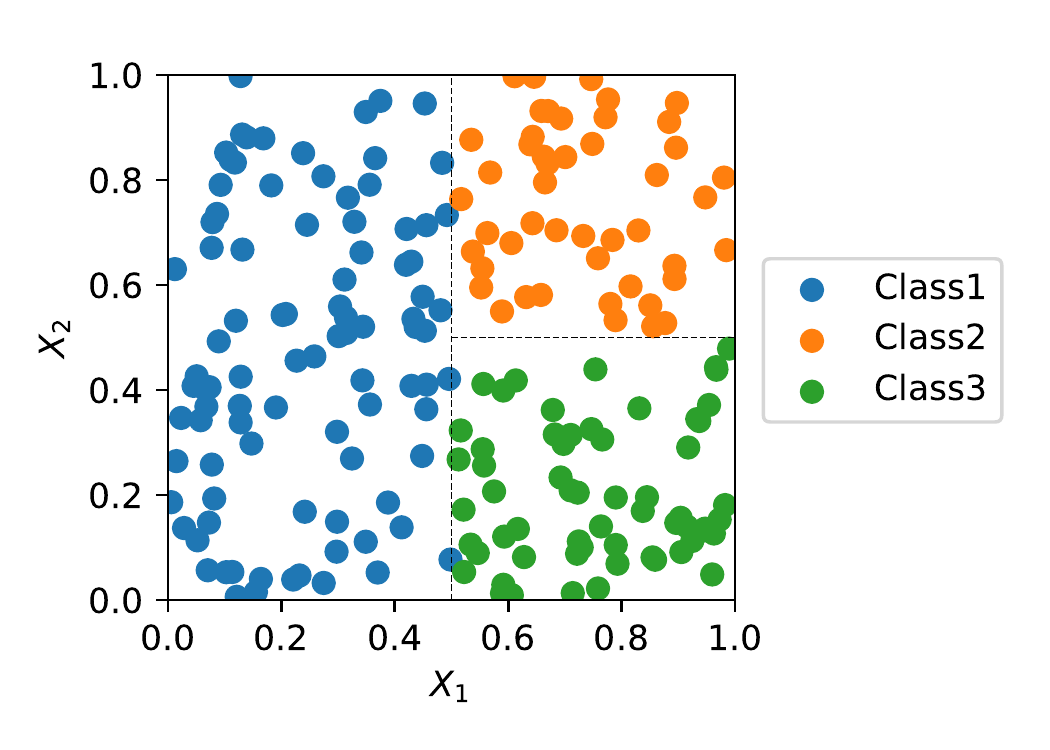}
	\end{minipage}}
	\hfill
	\subfigure[Decreasing validation error in the training process. \label{fig:3classerr}]{
		\begin{minipage}[h]{0.47\columnwidth}
			\centering
			\includegraphics[height=0.21\textheight]{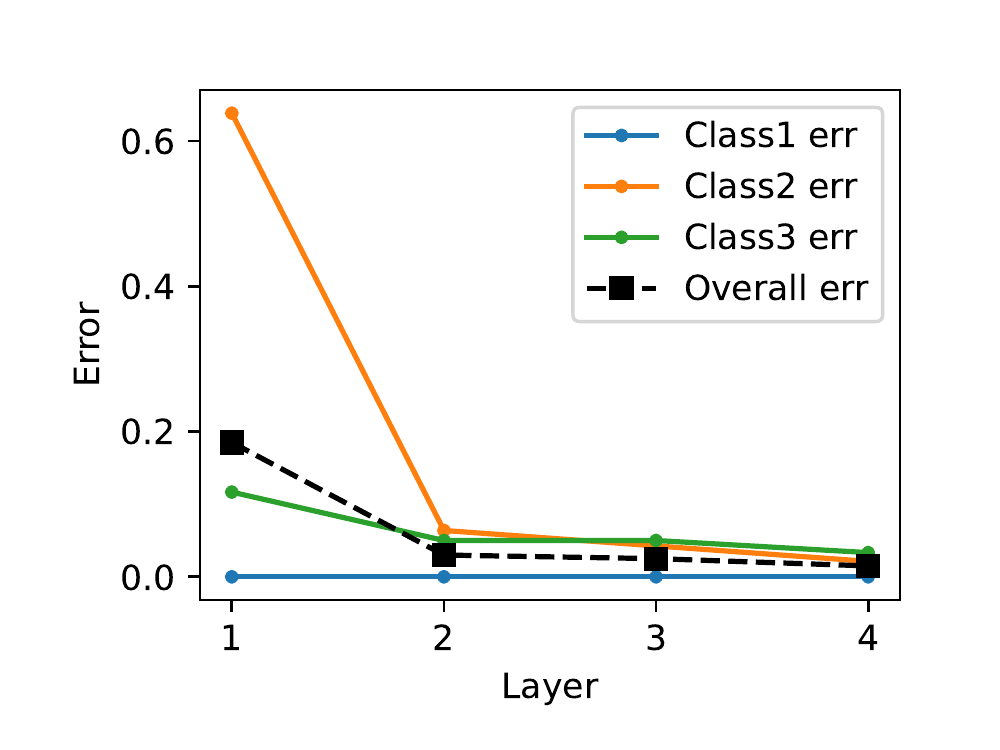}
	\end{minipage}}
	\vskip -0.1 in
	\subfigure[Prediction and feature contributions in the first layer. \label{fig:firstlayercontrib}]{
		\begin{minipage}[h]{0.47\columnwidth}
			\centering
			\includegraphics[width=\columnwidth]{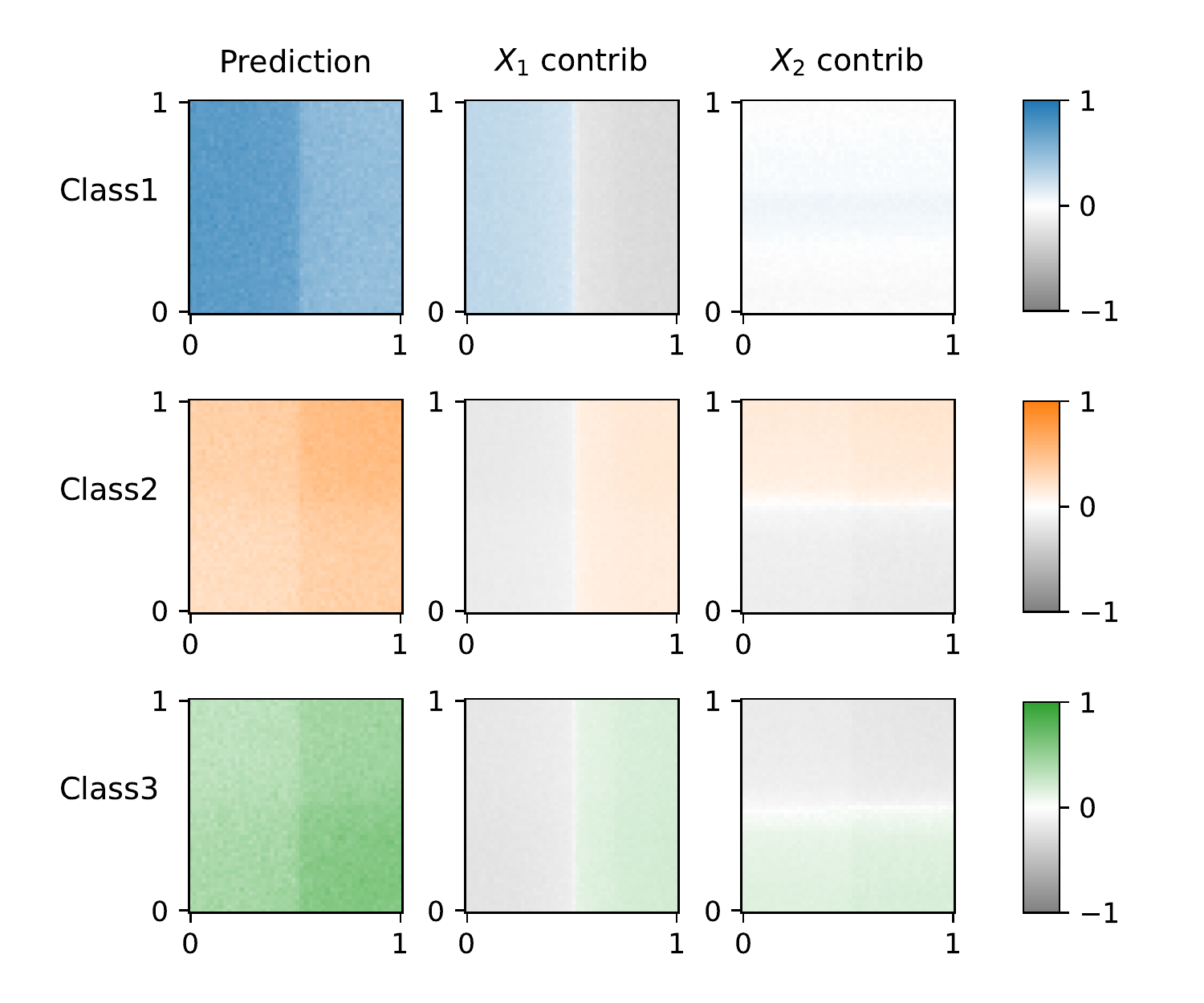}
	\end{minipage}}
	\subfigure[Prediction and feature contributions in the last layer. \label{fig:bestlayercontrib}]{
		\begin{minipage}[h]{0.47\columnwidth}
			\centering
			\includegraphics[width=\columnwidth]{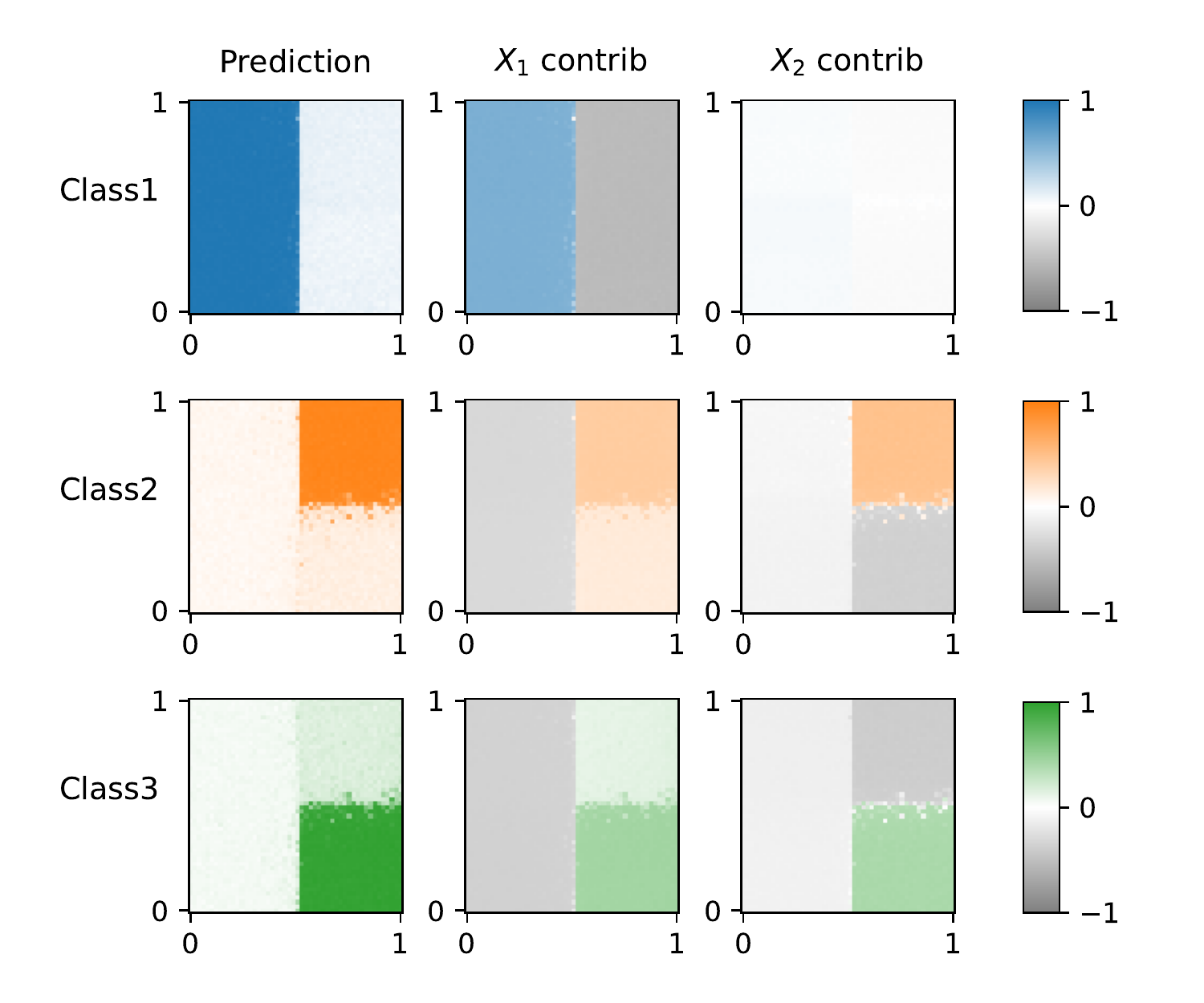}
	\end{minipage}}
	\vskip -0.05 in
	\caption{Feature contributions for the 3-class classification problem. As shown in (b) and the prediction figures in (c) (d), the predictive performance improves from the first layer to the last layer. The learned feature contributions are coarse and ambiguous in (c), while clear and precise in (d). 
		The colored areas in the contribution plots indicate an increase in the probabilities of data in these areas belonging to the corresponding class in the areas, and the gray areas indicate a decrease in probability.}
	\label{fig:contribcls}
\end{figure*}
\begin{table}[!t]
	\centering
	\begin{tabular}{c|c|c|c|c}
		\toprule
		$\vx$ & $\mu_0$ & $\widetilde{f}_{L,1}(\vx)$ & $\widetilde{f}_{L,2}(\vx)$ & $f_L(\vx)$ \\
		\midrule
		$(0,1)$ & $[\textbf{0.46}, 0.24, 0.30  ]$ & $[\textbf{0.37}, -0.14, -0.16]$ & $[\textbf{0.00}, 0.00, -0.06]$ & $[\textbf{0.88}, 0.08,0.04]$ \\
		$(1,1)$ & $[0.46, \textbf{0.24}, 0.30  ]$ & $[-0.36, \textbf{0.24}, 0.03]$ & $[0.00, \textbf{0.28}, -0.21]$ & $[0.05, \textbf{0.88}, 0.07]$ \\
		\bottomrule
	\end{tabular}
	\caption{Example of feature contributions of two sample points for the 3-class classification problem. Only values related to the two relevant dimensions are presented. The values associated with the predicted class are shown in bold.}
	\label{tbl:2points}
\end{table}

\paragraph{Classification}
We generate a 3-class classification problem. Figure~\ref{fig:3classdata} shows the training data in the first two dimensions. We add another 100 dimensions of irrelevant features uniformly valued between 0 and 1 to simulate background noise. There are 200 training samples and 2500 test samples.
With Figure~\ref{fig:3classerr} confirming us that the cascade forest structure improves performance layer by layer, Figure~\ref{fig:firstlayercontrib} and \ref{fig:bestlayercontrib} visualize the feature contribution in the first layer and the last layer to unveil what extra information the model has learned in the cascade layers. 
We plot for three classes separately, using the corresponding color with different darkness to indicate predictive probabilities. Then there are two contribution plots for $X_1$ and $X_2$ respectively. The colored areas indicate the corresponding feature enlarges the probabilities of data in these areas belonging to the corresponding class, and the gray areas versa.

We can see that the first layer's feature contribution is vague and of light color, while the last layer feature contribution is of darker color, leading to clearer peak probabilities in the prediction. The refinement of feature contribution matches the improvement of performance, and also better captures the underlying data generating scheme.
We can also observe that deep forest learns feature interaction in deep layers. Different from the simulated regression data, there actually is feature interaction in generating the classification data. However, in the first layer, the change of feature contribution is mainly along the considered feature. 
But in the last layer, the influence of the other feature emerges in the plots, showing that the cascade structure enables deep forest to learn details on the two relevant features. 

Given the last layer's feature contribution, we can easily tell how does deep forest make prediction. For example, the point $(0,1)$ has the highest predicted probability of belonging to Class 1, and the contribution to the Class 1's probability is mainly due to $X_1$. The point $(1,1)$ has the has the highest predicted probability of belonging to Class 2, and the contribution comes both from $X_1$ and $X_2$. Table~\ref{tbl:2points} shows the detailed values of decomposing the predicted probabilities of these two points into feature contributions. Only the contributions of the first two features are reported, and the other 100 irrelevant features also contribute to the prediction, but their contributions are small.

\subsection{Illustration of feature importance for deep forest}\label{sec:exp-MDI}
In this section, we show the computation of feature importance for deep forest. We design synthetic regression and classification data sets so that we can tell the relative importance of each feature. The hyper-parameter settings remain the same as in the previous section.
\paragraph{Regression}
we generate a regression problem with different coefficients for each feature and no interaction between features so that we can easily check whether the feature importance is reasonable. 
We generate 1000 training and test samples each according to the data generation function 
\begin{equation}
	f(\vx) = \sum_{k=1}^K k x_k\ .
\end{equation}

We measure the MDI quality using the percentage of correctly ranked pairs by relative estimated importance. 
Figure~\ref{fig:reg_auc} shows that as the test error decreases layer by layer, the ranking quality of relative MDI feature importance increases. 
As we would expect, deep forest learns more and improves performance through the cascade structure, so MDI would be more accurate in the deep layers. Figure~\ref{fig:reg_impt} directly compares the estimated MDI in the first layer and the last layer. We can see a clear refinement of the estimated MDI towards the underlying data generating mechanism.

\begin{figure}[t]
	\subfigure[Decreasing test error and increasing MDI quality measured by the percentage of correctly ranked pairs.\label{fig:reg_auc}]{
		\begin{minipage}[h]{0.35\columnwidth}
			\centering
			\includegraphics[height=0.18\textheight]{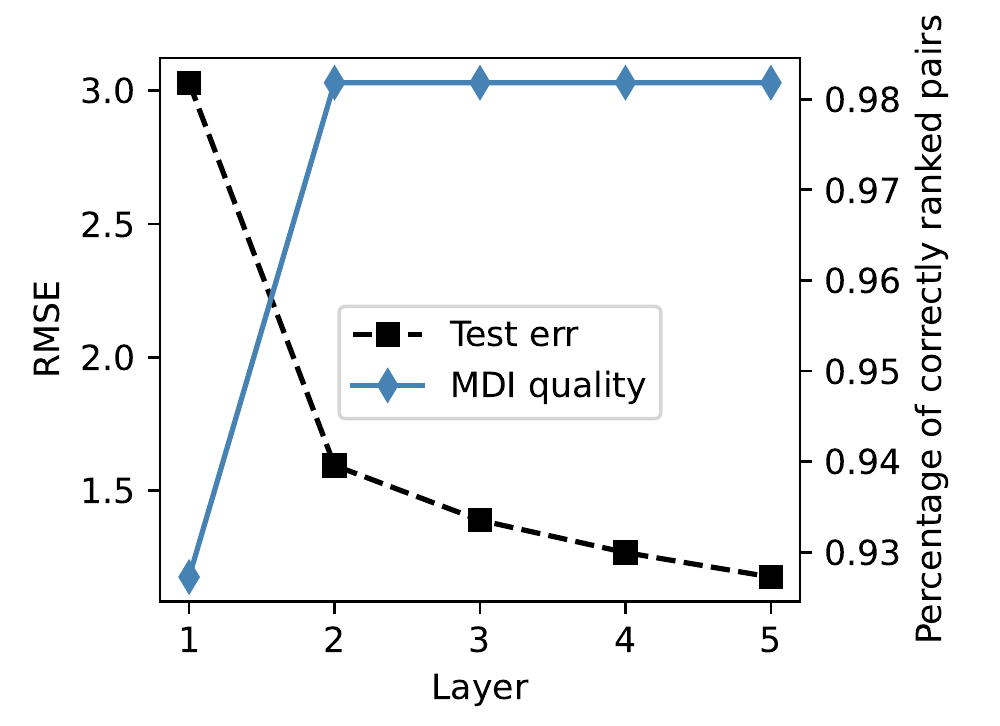}
	\end{minipage}}
	\hfill
	\subfigure[Comparison of MDI of the first layer and the last layer in deep forest. The MDI in the last layer has an obvious refinement over the result in the first layer. \label{fig:reg_impt}]{
		\begin{minipage}[h]{0.6\columnwidth}
			\centering
			\includegraphics[height=0.18\textheight]{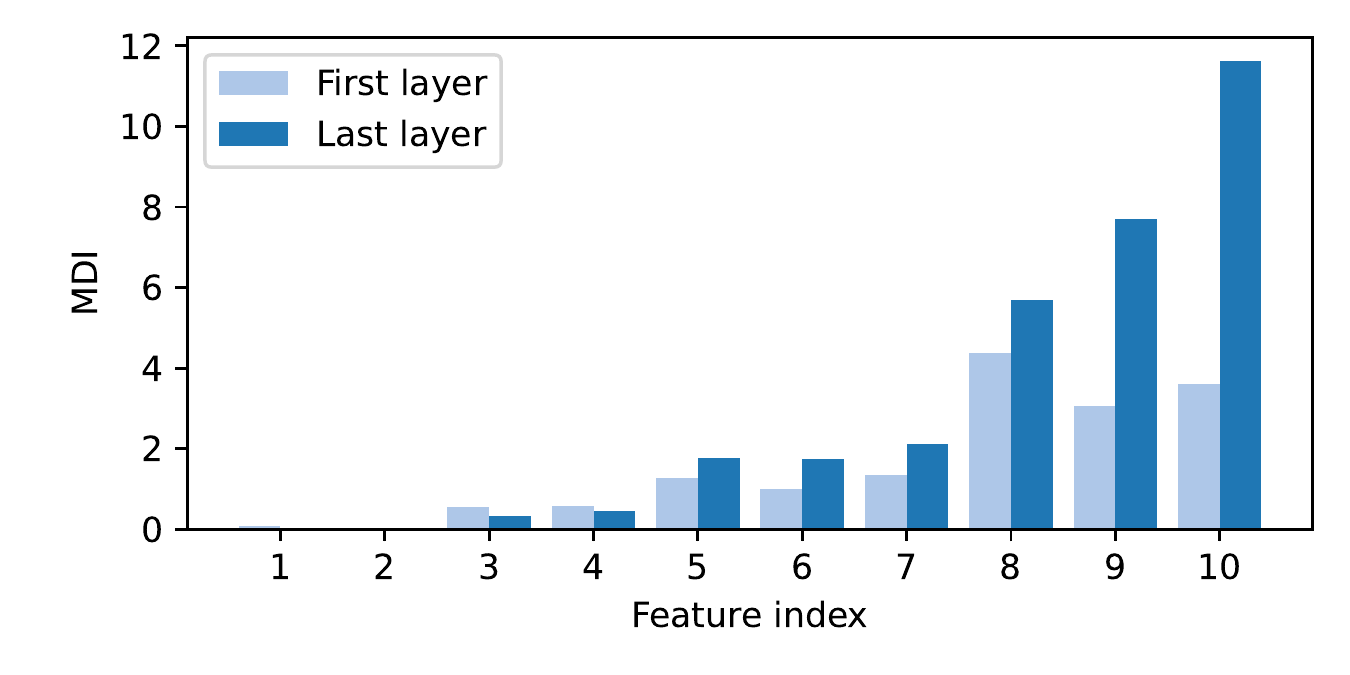}
	\end{minipage}}
	\caption{With the growing of layers in deep forest, the estimated MDI catches the underlying importance of features more accurately.}
	\label{fig:featureimportance}
\end{figure}
\begin{figure}[t]
	\subfigure[MDI of the two relevant features and the average MDI of the irrelevant features with increasing layers.\label{fig:3classimpt1}]{
		\begin{minipage}[h]{0.35\columnwidth}
			\centering
			\includegraphics[height=0.18\textheight]{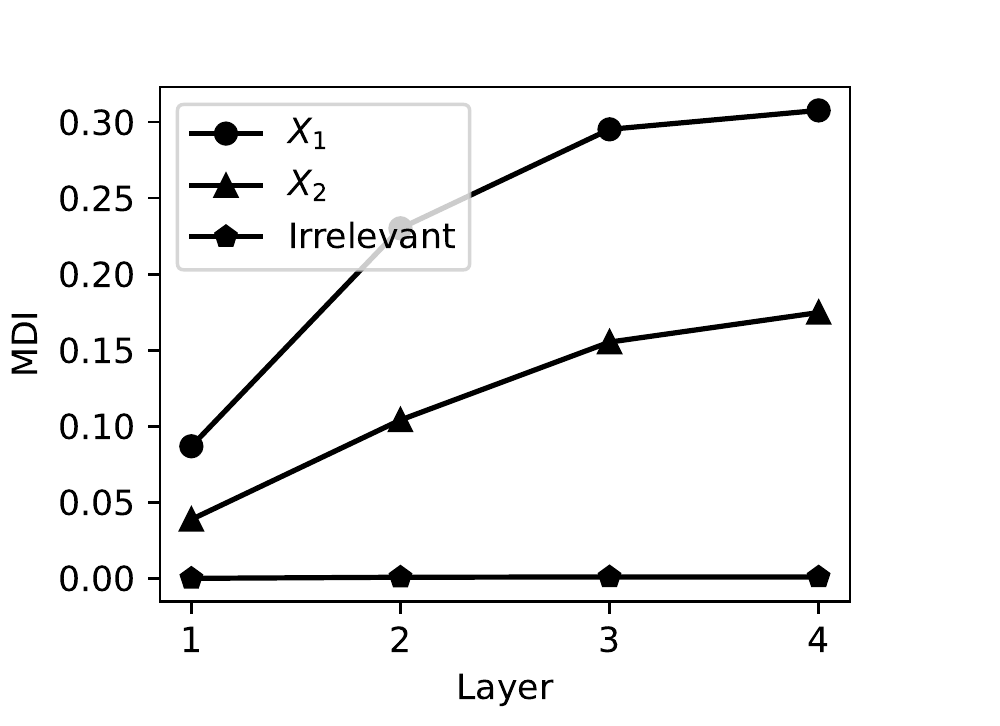}
	\end{minipage}}
	\hfill
	\subfigure[Local MDI in deep forest model with respect to each class. $X_1$ and $X_2$ play different roles in different classes. The average MDI for irrelevant features is always zero.\label{fig:3classimpthist}]{
		\begin{minipage}[h]{0.6\columnwidth}
			\centering
			\includegraphics[height=0.18\textheight]{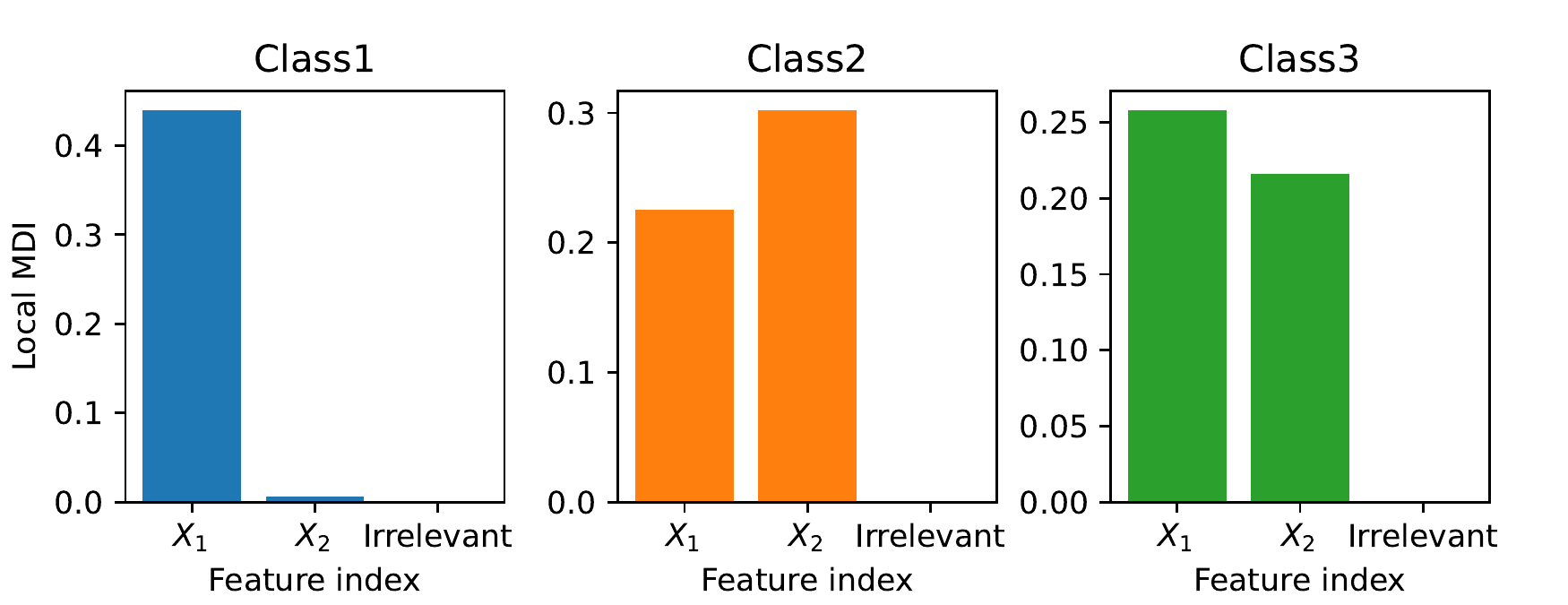}
	\end{minipage}}
	\caption{Overall MDI and local MDI for the 3-class problem.
	}
	\label{fig:3classimpt_cla}
\end{figure}

\paragraph{Classification}
We turn back to the 3-class classification problem as in Figure~\ref{fig:contribcls}. 
Figure~\ref{fig:3classimpt1} shows the change of global MDI with respect to layers. We can see that 
\begin{itemize}
	\item The global MDI of $X_1$ and $X_2$ both increase layer by layer, indicating deep forest gradually mines more information from the relevant features.
	\item $X_1$ has a higher MDI than $X_2$, which matches the data generating process, since using $X_1$ alone can separate all the instances from Class 1 which amounts to half of the data distribution.
	\item The average MDI of the 100 irrelevant features is always close to zero, indicating that the model has a good ability to distinguish relevant features from irrelevant features. 
\end{itemize}

Figure~\ref{fig:3classimpthist} shows the local MDI calculated for each class. Since equation~\eqref{eq:covrelationship} enables us to compute feature importance by summing over instances, we can calculate local feature importance for each class as 
\begin{equation*}
	\label{eq:local}
	\widehat{\operatorname{MDI}}_c(k,DF) 
	= \frac{1}{\left|\{i:y_i=c\}\right|} \sum_{y_i =c} \widetilde{f}_{L, k}\left(\boldsymbol{x}_{i}\right) \cdot \left[\mathbb{I}\left(y_i=1\right), \mathbb{I}\left(y_i=2\right), \mathbb{I}\left(y_i=3\right)\right], \ c=1,2,3,
\end{equation*}
From Figure~\ref{fig:3classimpthist} we can observe that, for Class 1, only feature $X_1$ matters, and feature $X_2$ has very low feature importance, while for Class 2 and Class 3, both $X_1$ and $X_2$ play very important roles. The above observation clearly matches the data generating mechanism.

\subsection{Comparison of feature importance ranking quality} \label{sec:exp-rank}
In this section, we compare our MDI feature importance for deep forest, named \textbf{MDI(DF)}, to the following methods.
\begin{itemize}
	\item[--] \textbf{MDI(RF)}: the mean decrease of impurity by splitting on the given feature during the training process of random forest~\cite{breiman1984classification}. But it tends to overestimate the feature importance of irrelevant features~\cite{li2019debiased}. 
	\item[--] \textbf{MDI-oob(RF)}: a debiased version of MDI using out-of-bag samples for random forest~\cite{li2019debiased}. 
	\item[--] \textbf{MDA(RF)}: the mean decrease in accuracy of a trained random forest caused by randomly permuting the values of the given feature. 
	\item[--] \textbf{MDA(DF)}: the mean decrease in accuracy of a trained deep forest caused by randomly permuting the values of the given feature. 
\end{itemize}

In each layer of deep forest, there are still 4 forests, each with 50 trees in it. Accordingly, we grow 200 trees in random forest. Since we are dealing with more complicated data sets, the depth of trees in deep forest is fixed to 8, and the trees in random forests are fully grown.  

The simulation data set \textit{sim} is generated according to \cite{li2019debiased}. It has 50 features, with the $j^{th}$ feature taking random value from $0,1,...,j$ with equal probability. We randomly choose a set $S$ of $5$ features from the first ten features as relevant features. 
The labels are generated according to $P(Y=1 \mid X)=\operatorname{Logistic}\left(\frac{2}{5} \sum_{j \in S} X_j / j-1\right)$.
We generate 1000 i.i.d. training samples and 1000 i.i.d. validation samples. The validation samples are only used for MDA based mehtods. 

The five benchmark data sets are processed with two steps. First, we copy a data set's feature matrix but randomly permute values of each feature, and then concatenate them to the original feature matrix. By this way, the feature set is now half relevant and half irrelevant. Second, since identifying irrelevant features is a relatively easier task compared with prediction, we reduce the number of training samples to avoid the situation where all the methods are equally perfect. We reduced the number of training samples of data set \textit{vehicle}, \textit{segment} and \textit{phishing} to 20\%, for data set \textit{phishing} and \textit{pendigits} we reduce to 10\%. We also keep a separate validation set for MDA based methods, whose number of samples is the same as the training set. 

\begin{table}[!t]
	\centering
	\begin{tabular}{c|ccccc}
		\toprule
		& MDI(RF) & MDI-oob(RF) & MDA(RF) & MDA(DF) & MDI(DF)\\
		\midrule
		\textit{sim} & 0.14 & 0.80 & 0.70 & 0.72 & \textbf{0.82} \\
		\textit{vehicle} & 0.95 & 0.92 & 0.68 & 0.64 & \textbf{0.99} \\
		\textit{segment} & 0.86 & 0.93 & 0.65 & 0.69 & \textbf{0.95} \\
		\textit{phishing} & 0.58 & 0.96 & 0.70 & 0.68 & \textbf{0.97} \\
		\textit{pendigits} & \textbf{1.0} & \textbf{1.0} & 0.98 & 0.99 & \textbf{1.0} \\
		\textit{satimage} & \textbf{1.0} & \textbf{1.0} & 0.60 & 0.59 & \textbf{1.0} \\
		\bottomrule
	\end{tabular}
	\caption{Average AUC scores of 20 runs for relevant feature identification. The best result of each data set is shown in bold.}
	\label{tbl:rank}
\end{table}

\begin{figure}[!t]
	\centering
	\includegraphics[width=0.65\linewidth]{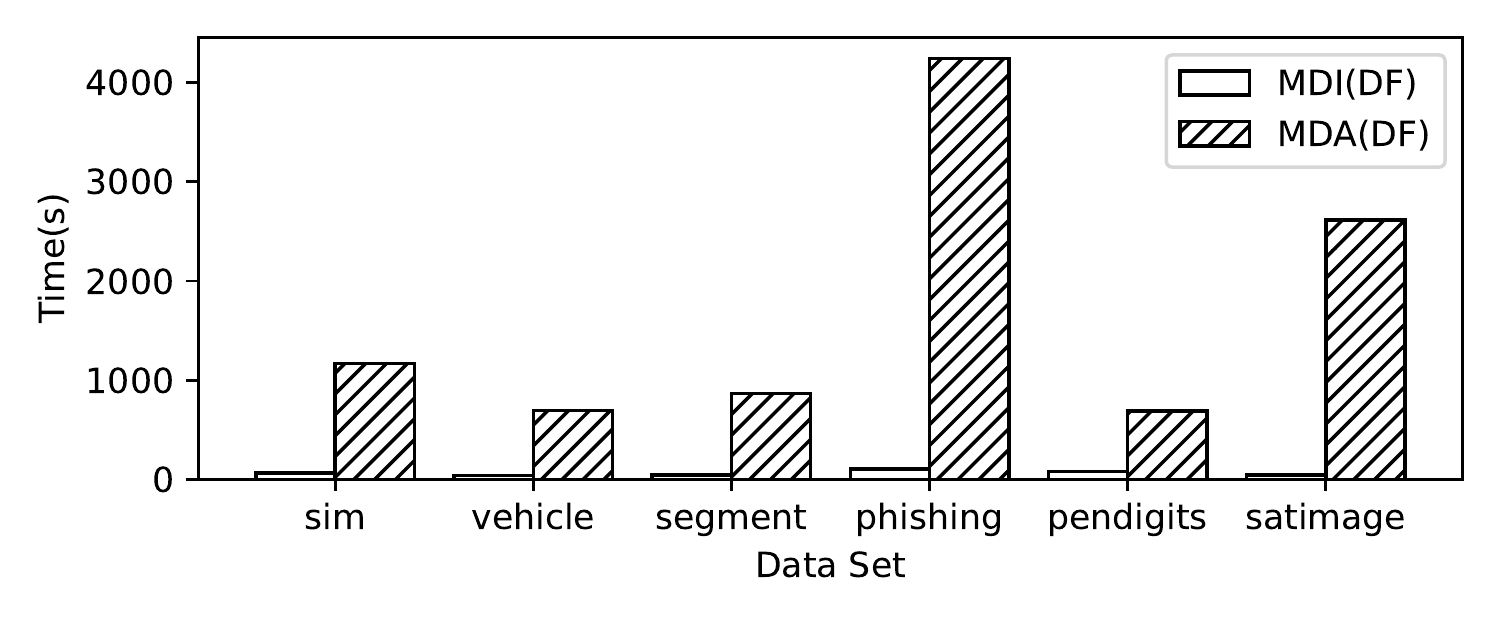}
	\caption{Running time comparison of our MDI method to that of MDA for deep forest. Each experiment is repeated 20 times and the average running time is reported.}
	\label{fig:runningtime}
\end{figure}

Table~\ref{tbl:rank} reports the average AUC of identifying relevant features. 
We can find that MDI(DF) is better than all the RF based methods, namely MDI(RF), MDI-oob(RF), and MDA(RF). Meanwhile, it is also much better than MDA(DF). This is because these generated and manipulated data sets have low accuracy, which hinders the performance of accuracy based method. Furthermore, MDI(DF) is not only better than MDA(DF) in the AUC value of identifying related features, but also requires less running time as shown in Figure~\ref{fig:runningtime}. The reported running time includes model training and feature importance computation. MDI(DF) can compute feature contribution and feature importance during training process, which lead to only mild increasing of training time. In contrast, MDA(DF) needs to call the test process repeatedly for each feature, and usually for multiple times in order to get a reliable estimation of the accuracy decrease. 

\subsection{Comparison of different calibration methods}\label{sec:exp-calib}
To evaluate the effectiveness of using partial additive calibration as Eq.~\eqref{eq:partial} in the step of calculating feature contribution, we compare MDI(DF) to that using naive multiplicative calibration as Eq.~\eqref{eq:multiplicative}, denoted by \textbf{MDI(DF)$^*$}, and that using naive multiplicative calibration as Eq.~\eqref{eq:additive}, denoted by \textbf{MDI(DF)$^+$}.
The data sets and other settings remain the same as the previous section. And we still use feature importance ranking quality as the evaluation criterion. 

\begin{table}[!t]
	\centering
	\begin{tabular}{c|ccc}
		\toprule
		& MDI(DF)$^*$ & MDI(DF)$^+$ & MDI(DF) \\
		\midrule
		\textit{sim} & 0.80 & \textbf{0.82} & \textbf{0.82} \\
		\textit{vehicle} &  0.84 & 0.98 & \textbf{0.99} \\
		\textit{segment} &  0.93 & 0.94 & \textbf{0.95} \\
		\textit{phishing} & 0.96 & 0.96 & \textbf{0.97} \\
		\textit{pendigits} & 0.98 & \textbf{1.0} & \textbf{1.0} \\
		\textit{satimage} & 0.93 & \textbf{1.0} & \textbf{1.0} \\
		\bottomrule
	\end{tabular}
	\caption{Comparison of different calibration methods based on average AUC scores for relevant feature identification. The best result of each data set is shown in bold.}
	\label{tbl:calibration}
\end{table}

Table~\ref{tbl:calibration} shows that MDI(DF), which uses partial additive calibration, achieves best AUC in all the data sets. The performance of MDI(DF)$^+$, which uses naive additive calibration, follows closely behind. The comparison results justifies the our design of calibration method.

\subsection{Application on bike sharing task}\label{sec:bikesharing}
This section does not aim to compare, instead it provides an example with a real-world bike sharing task~\cite{bikesharing}, showing how our deep forest interpretation tools can be used in an application. 

The bike sharing data set contains 731 days of records. For clearer demonstration, we process the data set and only keep six features. This is a regression task with a single prediction value, so the feature contributions are scalars. The data are randomly split into training and test data sets in a 2:1 ratio. 

The explanatory results of each prediction help the end user understand how the deep forest model works, enabling more reliable deployment of model and helping to find out why a prediction is unexpected. For example, table~\ref{tbl:bikecontribution} shows the feature contribution of a trained deep forest on two days in the test set. We can see the platform doing better in 2012 than it did in 2011. In addition, the low temperature and rainy and windy weather on 2011/12/7 both contributed to a decrease in bike rental volumes. Mild temperature and nice weather have led to an increase in bike rental volumes on 2012/9/21. 
Figure~\ref{fig:bikeMDI} shows the MDI feature importance of deep forest. We can see that temperature is the most important feature. In addition, the year is also an important feature, which shows that the platform has improved a lot in two years. Other weather conditions also play a role, but whether it is a working day has little impact on bike rental volumes.

\begin{table}[!t]
	\centering
	\begin{tabular}{cccccccc}
		\toprule
		date & & Year & isWorkingDay & isClearDay & Temperature & Humidity & WindSpeed\\
		\midrule
		\midrule
		\multirow{2}*{2011/12/7} & feature value & 2011 & 1 & 0 & 17$^{\circ}$C & 97\% & 18 km/h \\
		&feature contribution & -714 & 24 & -104 & -1479 & -244 & -136\\
		\midrule
		\midrule
		\multirow{2}*{2012/9/21} & feature value & 2012 & 1 & 1 & 25$^{\circ}$C & 67\% & 10 km/h \\
		&feature contribution & 1031 & 44 & 252 & 1117 & 145 & 74\\
		\bottomrule
	\end{tabular}
	\caption{Example of feature contribution of deep forest in the bike sharing task. Positive and negative contributions are based on the training label mean.}
	\label{tbl:bikecontribution}
\end{table}

\begin{figure}[!t]
	\centering
	\includegraphics[width=0.65\linewidth]{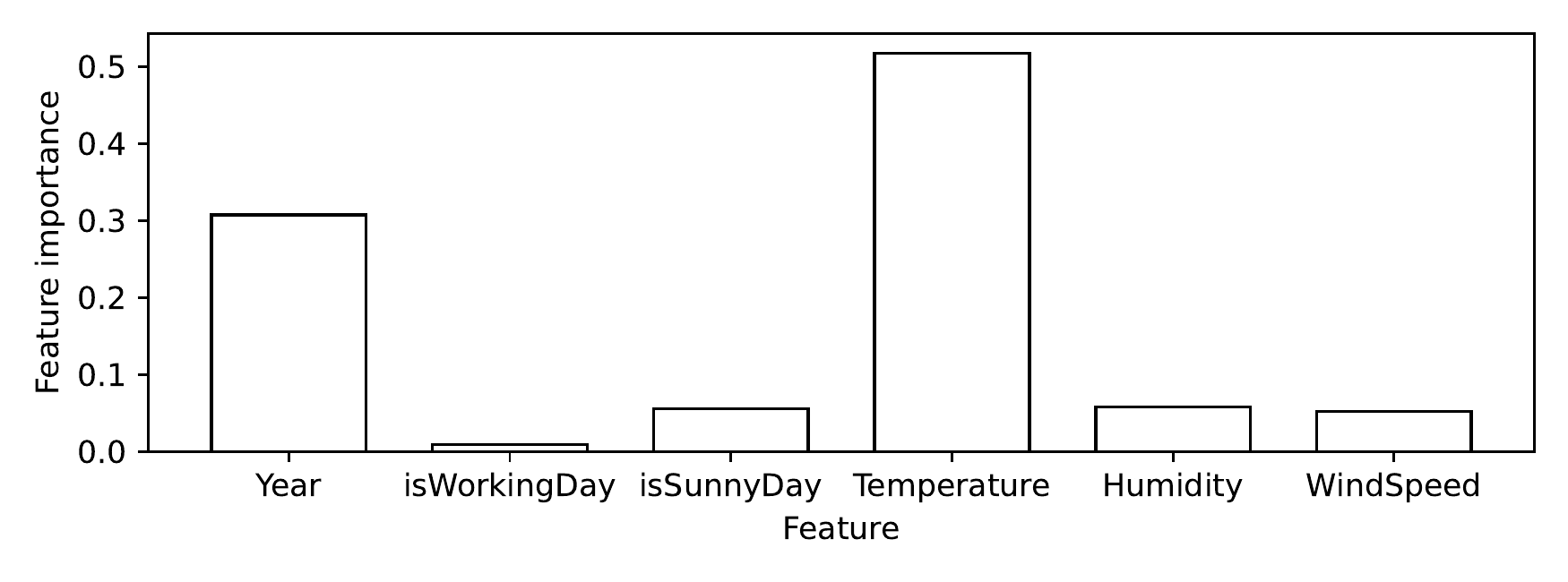}
	\caption{MDI feature importance of deep forest in the bike sharing task. The sum of feature importance is normalized.}
	\label{fig:bikeMDI}
\end{figure}

\section{Conclusion}\label{sec:conclusion}

In this paper, we propose two interpretation tools for deep forests, namely feature contribution and feature importance. The former decomposes every single prediction of deep forest into the contributions of each original feature. The latter characterizes the overall impact of a feature in the whole model. 
Experimental results show that they can not only faithfully reflect the influence of each feature on deep forest prediction, but also have a better characterization of feature importance than random forest. We believe that the interpreting tools we provide can further promote the application of deep forests, as well as deepen our understanding of the data and model.

\printbibliography



%
%

\end{document}